\documentclass{article}
\usepackage[numbers]{natbib}  
\newif\ifarxiv
\arxivtrue 

\ifarxiv 
    \usepackage[preprint]{neurips_2026} 
\else 
    \usepackage{neurips_2026} 
\fi

\usepackage[utf8]{inputenc} 
\usepackage[T1]{fontenc}    
\usepackage{hyperref}       
\usepackage{url}            
\usepackage{booktabs}       
\usepackage{amsfonts}       
\usepackage{nicefrac}       
\usepackage{microtype}      
\usepackage{xcolor}         

\usepackage{graphicx} 

\usepackage{algorithm}
\usepackage{algorithm}
\usepackage{algpseudocode}

\usepackage{amsmath}
\usepackage{multirow}     
\usepackage{colortbl}     

\usepackage[inline]{enumitem}

\usepackage{graphicx}
\usepackage{arydshln}
\usepackage{multirow}
\usepackage{subcaption}
\usepackage{booktabs} 
\usepackage{amsmath}
\usepackage{amssymb}
\usepackage{mathtools}
\usepackage{amsthm}

\usepackage{stmaryrd}
\usepackage{wrapfig}
\usepackage{bbold}
\usepackage{comment}
\usepackage{tabularx}
\usepackage{mathabx}

\usepackage{minitoc}

\definecolor{GNN}{HTML}{b85450}
\definecolor{SNN}{HTML}{82b366}
\definecolor{HNN}{HTML}{d6b656}
\definecolor{CWN}{HTML}{6c8ebf}
\definecolor{ours}{HTML}{FF77FF}

\theoremstyle{plain}
\newtheorem{thm}{Theorem}

\newtheorem{prop}{Proposition}
\newtheorem{cor}{Corollary}

\newtheorem{defi}{Definition}
\newtheorem{rmq}{Remark}

\newcommand{\ie}{\textit{i}.\textit{e}., }
\newcommand{\eg}{\textit{e}.\textit{g}., }

\title{Scaling Higher-Order Graph Learning with Maximal Clique Complexes}

\author{%
  Antoine~Vialle \\
  LTCI, Télécom Paris\\
  Institut Polytechnique de Paris\\ 91120 Palaiseau, France \\
  \texttt{antoine.vialle@telecom-paris.fr} \\
  \And Aref~Einizade \\
  SAMOVAR, Télécom SudParis \\
   Institut Polytechnique de Paris\\ 91120 Palaiseau, France \\
  \texttt{aref.einizade@telecom-sudparis.eu} \\
    \And Fragkiskos~D.~Malliaros \\  CentraleSupélec, Inria \\
Université Paris-Saclay \\ 91190 Gif-sur-Yvette, France \\ \texttt{fragkiskos.malliaros@centralesupelec.fr} \\
  \And Jhony~H.~Giraldo \\LTCI, Télécom Paris \\ Institut Polytechnique de Paris \\ 91120 Palaiseau, France \\ \texttt{jhony.giraldo@telecom-paris.fr}
}

\begin{document}

\maketitle

\begin{abstract}


Graph neural networks (GNNs) are limited to modeling pairwise interactions, while higher-order models based on cell complexes achieve greater expressivity but often suffer from poor scalability. 
We introduce simplified and factored cellular Weisfeiler–Leman tests (sCWL and fCWL), which preserve the expressivity of the CWL test while improving computational efficiency.
We further introduce the maximal clique complex, enabling scalable CWNs with reduced time and memory complexity while retaining strong empirical performance.
To avoid explicit clique enumeration, we propose CliqueWalk, a biased random walk that samples maximal cliques and scales linearly with graph size. 
These contributions yield a scalable topological learning framework for higher-order graph representation.
\end{abstract}

\section{Introduction}
Graphs provide a natural way to represent interactions between entities, and graph neural networks (GNNs) have become the standard approach for learning on such data~\citep{kipf2017semisupervisedclassificationgraphconvolutional, defferrard2017convolutionalneuralnetworksgraphs}. 
GNNs have achieved strong performance in diverse domains, including social network analysis~\citep{fan2019graphneuralnetworkssocial}, molecular property prediction~\citep{duvenaud2015convolutionalnetworksgraphslearning}, and computer vision~\citep{Krzywda_2022}. 
However, conventional GNNs are limited to modeling pairwise interactions between nodes, which constrains their ability to capture complex multi-way relationships \citep{battiloro2024generalized}.
To address this limitation, recent work explores higher-order structures such as simplicial complexes~\citep{ebli2020simplicialneuralnetworks, bodnar2021weisfeilertopo,einizade2025continuous}, cell complexes~\citep{hajij2021cellcomplexneuralnetworks, bodnar2021weisfeilercell}, and hypergraphs~\citep{feng2019hypergraphneuralnetworks}.

Hypergraphs generalize graphs by allowing edges, called hyperedges, to connect more than two nodes \citep{feng2019hypergraphneuralnetworks}. 
A hyperedge thus represents a group interaction, for example, a set of coauthors of the same paper in a co-authorship network~\citep{9782536}.
Beyond hypergraphs, cell complexes provide a general combinatorial framework that organizes higher-order structures~\citep{MR1867354}.
A cell complex contains cells of different dimensions: nodes (0-cells), edges (1-cells), triangles (2-cells), and so on \citep{bodnar2021weisfeilercell}.
Simplicial complexes are a special case of cell complexes in which all subsets of a cell are also included, ensuring closure under subset operations \citep{einizade2025continuous}.
In this setting, entities interact whenever they differ by the addition or deletion of a single node.

Several approaches lift graphs into higher-order structures, allowing the use of simplicial and cell complexes for learning tasks~\citep{bodnar2021weisfeilertopo, papillon2024architecturestopologicaldeeplearning, papamarkou2024position}.
One of these strategies is the clique lifting, where simplicial or cell complexes are built by including all cliques of the graph up to a fixed size (\eg edges or triangles) \citep{bodnar2021weisfeilercell}.
While effective for capturing higher-order information, these methods are often computationally expensive and require significant memory resources.
Furthermore, the clique problem is well-known to require algorithms with exponential runtime in the worst case \citep{cormen01introduction}.




In this paper, we address two central questions:
($i$) \textit{how to simplify cellular Weisfeiler networks (CWNs) without sacrificing expressivity}, and
($ii$) \textit{how to exploit clique-based higher-order structure in a scalable manner}.

\begin{wrapfigure}{r}{0.45\textwidth}
    \includegraphics[width=0.45\textwidth]{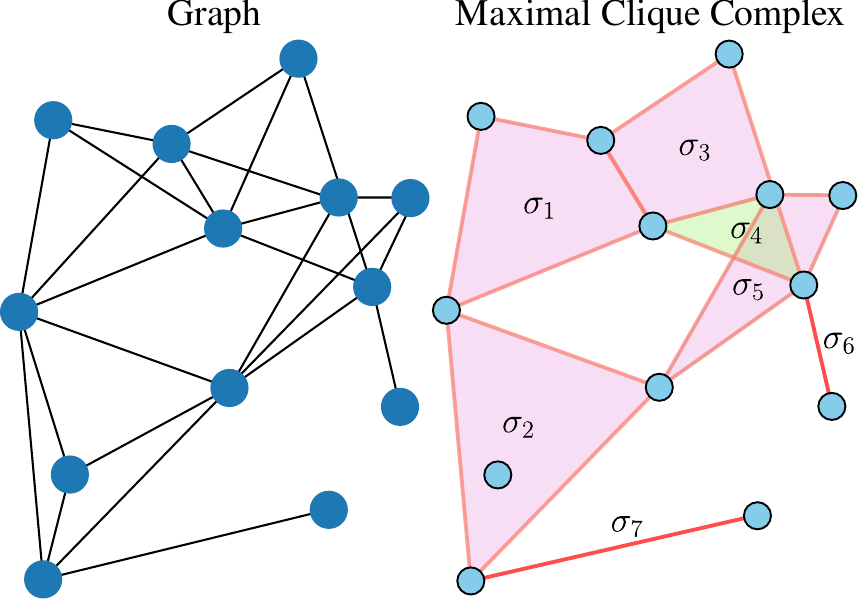}
    \caption{Maximal clique complex.}
    \vspace{-12pt}
\label{fig:maximum_clique_complex}
\end{wrapfigure}

To address the first question, we introduce the simplified and factored cellular Weisfeiler–Leman tests (sCWL and fCWL), together with their corresponding neural architectures (sCWNs and fCWNs). 
We show that these variants preserve the full expressive power of the original CWL test of \cite{bodnar2021weisfeilercell}, while exhibiting improved scalability properties.
These simplified architectures enable the use of more efficient higher-order constructions in specific graph regimes.
In particular, in graphs with few large, weakly overlapping cliques, clique-based aggregation is more scalable than classical message passing, reducing quadratic intra-clique message passing to linear complexity via a node–clique–node scheme.


Motivated by this observation and to answer the second question, we propose the \textit{maximal clique complex}, a simplified cell complex that encodes only the maximal cliques of the graph (Figure~\ref{fig:maximum_clique_complex}). 
Since enumerating all maximal cliques is computationally expensive and becomes infeasible for large graphs, we further introduce CliqueWalk, a biased random-walk procedure that efficiently samples maximal cliques and scales linearly with the number of edges.
The sampled cliques define the higher-order cells used in our architectures, enabling scalable models with competitive performance on node and graph classification benchmarks.





The main contributions of this paper are:
\begin{itemize}[leftmargin=0.5cm]
    \itemsep0em
    \item We introduce the sCWL and fCWL tests and prove that they are as expressive as the regular CWL test, while offering better scaling properties.
    \item We present the maximal clique complex, a simplified higher-order structure encoding maximal cliques, and show that the resulting sCWN and fCWN are more memory- and computationally efficient than standard CWNs, while preserving expressivity.
    \item Since enumerating maximal cliques can take exponential time, we propose CliqueWalk, a biased random walk that efficiently samples maximal cliques and scales linearly with the number of edges, enabling clique-based methods on large graphs.
    \item We show competitive performance on graph and node classification benchmarks, matching existing GNNs while achieving gains in scalability and efficiency {compared to expressive baselines.}
\end{itemize}

\section{Related Work}


The expressive power of GNNs has been extensively studied, particularly through their ability to distinguish non-isomorphic graphs \citep{xu2018how, morris2021weisfeilerlemanneuralhigherorder}. It is well established that message-passing GNNs with injective aggregation are as powerful as the 1-Weisfeiler–Lehman (1-WL) test \citep{xu2018how}, and architectures such as the graph isomorphism network (GIN) are explicitly designed to match this expressivity.


However, these models remain limited in capturing higher-order interactions, as they rely on local aggregation over pairwise edges \citep{Bouritsas_2023, Feng2022HowPA}. To address this, several approaches have been proposed. Some extend the Weisfeiler–Lehman hierarchy to higher orders, yielding k-WL-based GNNs with increased expressive power but much higher computational complexity \citep{morris2021weisfeilerlemanneuralhigherorder, Maron2019}. Others enrich the graph structure by modeling higher-order relations directly. For instance, simplicial neural networks \citep{bodnar2021weisfeilertopo} operate on simplicial complexes, enabling message passing over higher-dimensional simplices and expressivity beyond 1-WL. Similarly, CWN \citep{bodnar2021weisfeilercell} generalizes this idea to arbitrary cell complexes, with message passing through boundary, co-boundary, and adjacency relations. These methods are formalized by the Cell-WL (CWL) test, which is strictly more expressive than 1-WL in specific settings and has shown strong empirical performance, particularly in molecular graph learning \citep{bodnar2021weisfeilercell, giusti2023cinenhancingtopologicalmessage}.


Despite these theoretical advances, simplicial and cell complex models suffer from limited scalability. Constructing higher-order complexes often requires enumerating large numbers of cliques, leading to prohibitive memory and time costs and preventing their efficient application to large-scale graphs.
In contrast, we introduce the maximal clique complex as a simplified higher-order structure that preserves CWL-level expressivity while enabling efficient clique-based neural architectures. 
Combined with our CliqueWalk sampling strategy, this yields a scalable approach to higher-order graph learning with strong theoretical guarantees and competitive performance.

\section{Preliminaries}

\textbf{Notation.}
Calligraphic letters denote sets, with $\lvert \mathcal{X} \rvert$ their cardinality. Lowercase boldface letters (e.g., $\mathbf{x}$) denote vectors.
$\bigoplus$ denotes a mapping from sets of vectors to a vector, \eg aggregation functions.


\begin{defi}[{Regular cell complex~\citep{Hansen_2019, bodnar2021weisfeilercell}}]
A regular cell complex is a topological space $X$ that can be divided into a collection of subspaces $\{X_\sigma\}_{\sigma \in \mathcal P_X}$, called \textbf{cells}, where $\mathcal P_X$ is the set of cells induced by the topological space $X$.  
These cells satisfy the following properties:
\begin{itemize}[leftmargin=0.5cm]
    \itemsep0em
    \item Every $x \in {X}$ has an open neighborhood that intersects only a finite number of cells.
    \item For any two cells $ X_\sigma$ and $ X_\tau$, $X_\tau \cap \overline{X_\sigma}\ne \emptyset$, if and only if $X_\tau$ is contained in $\overline{X_\sigma}$.
    \item Each cell is topologically equivalent (homeomorphic) to $\mathbb{R}^n$ for some dimension $n$.
    \item For each $\sigma \in \mathcal P_X$, there exists a homeomorphism $\varphi$ from a closed ball in $\mathbb{R}^{n_\sigma}$ onto $\overline{X_\sigma}$, where the restriction of $\varphi$ to the interior of the ball gives a homeomorphism onto the interior of $X_\sigma$.
\end{itemize}
\end{defi}
A graph $G=(\mathcal{V},\mathcal{E})$ can be interpreted as a special case of a cell complex, where vertices $\mathcal{V}$ and edges $\mathcal{E}$ correspond to $0$-cells and $1$-cells, respectively.

\begin{defi}[Cell complex adjacencies \citep{bodnar2021weisfeilercell}]
\label{def:adjacencies}
Let ${X}$ be a cell complex and $\sigma \in \mathcal{P}_X$ a cell.  
We define the following adjacency relations:
\begin{itemize}[leftmargin=0.5cm]
    \itemsep0em
    \item {Boundary cells $\mathcal{B}(\sigma)$:} cells that make up the boundary of $\sigma$ (e.g., the vertices of an edge).  
    \item Co-boundary cells $\mathcal{C}(\sigma)$: cells having $\sigma$ in their boundary (e.g., edges incident to a vertex).
    \item {Lower adjacent cells $\mathcal{N}_\downarrow(\sigma)$:} cells that share a boundary (e.g., edges meeting at a common node).
    \item Upper adjacent cells $\mathcal{N}_\uparrow(\sigma)$: cells that share a co-boundary (e.g., two nodes linked by an edge).
\end{itemize}
\end{defi}

\textbf{WL test.} 
A key challenge in graph theory is the graph isomorphism problem, which concerns deciding whether two graphs have the same structure.
Finding exact solutions is often computationally demanding, so faster approximate techniques, such as graph hashing, are commonly employed.
A classical and widely used technique for graph isomorphism test is the WL test \citep{weisfeiler1968reduction}.
The WL test provides an efficient heuristic for the graph isomorphism problem.
The formal definition of the WL test is provided in Appendix \ref{app:WL_Test}.
Beyond graphs, it can be naturally extended to regular cell complexes, capturing richer combinatorial structures.

\begin{defi}[CWL test \citep{bodnar2021weisfeilercell}]
\label{defi:cwl}
Let $X$ be a regular cell complex.  
The CWL test is defined as:
\begin{itemize}[leftmargin=0.5cm]
    \itemsep0em
    \item {Initialization:} All cells $\sigma \in \mathcal P_X$ are assigned the same initial color.
    \item {Color refinement:} At iteration $t+1$, the color of each cell $\sigma$ is updated according to $c_\sigma^{t+1} = \text{HASH}(c_\sigma^t,\, c_{\mathcal{B}(\sigma)}^t,\, c_{\mathcal{C}(\sigma)}^t,\, c_{\mathcal{N}_\downarrow(\sigma)}^t,\, c_{\mathcal{N}_\uparrow(\sigma)}^t)$, where $\text{HASH}$ is an injective function combining the color of $\sigma$ with those of its boundary, co-boundary, and adjacent cells.
    \item {Termination:} The process is repeated until the coloring stabilizes.  
    Two cell complexes are considered non-isomorphic if their color histograms differ.
\end{itemize}
\end{defi}


The CWL test is invariant under cell-complex isomorphisms. Given a graph-to-cell-complex map that preserves isomorphisms, CWL can be used to test graph isomorphism, see Figure~\ref{fig:CWL_ex} for an example. This corresponds to the cellular lifting map of \cite{bodnar2021weisfeilercell} (their Definition~8). Similarly, CWL relates to the WL test under skeleton-preserving lifting maps:
\begin{defi}[Skeleton preserving lifting \citep{bodnar2021weisfeilercell}]
    A lifting map $f(\cdot)$ is skeleton-preserving if for any graph $G = (\mathcal V, \mathcal E)$: ($i$) $f(G)$ contains $\mathcal V$ and $\mathcal E$ as cells, ($ii$) the cell complex $f(G)$ restricted to node and edge sets is isomorphic to $G$, \ie the incidence matrices of $G$ and $f(G)$ are equal up to permutation.
    
    
\end{defi}
The CWL scheme is more expressive than the standard WL test for skeleton-preserving lifting maps~\citep{bodnar2021weisfeilercell}. 
We next introduce a new cell test that also achieves better expressivity than the WL test without requiring such lifting maps, along with the structures studied in this work.
All proofs are provided in Appendix~\ref{app:proofs}.



\section{{Scaling Cell Complex Models and Maximal Cliques}}


\subsection{Cell Complex Expressivity Theory}


\citet{bodnar2021weisfeilercell} show we can simplify the CWL test while retaining the same expressivity.
\begin{thm}[\citet{bodnar2021weisfeilercell}] 
\label{thm:expressivity_cwl}
The CWL update rule restricted to {boundary} and {upper adjacency} messages is equivalent in expressive power to the full CWL update rule.
\end{thm}
We also demonstrate that a different simplified version retains the same expressivity.
\begin{thm} \label{thm:expressivity_scwl}
The CWL update rule restricted to boundary and co-boundary messages, called the simplified CWL (sCWL) test, is equivalent in expressive power to the full CWL update rule.
\end{thm}
This restricted scheme is useful in practice, as it leads to more computationally efficient models.

We also introduce a new test on cell complexes that, while keeping the node structure, enables at least the same expressivity as the CWL, sCWL, and WL tests.
\begin{defi}[Factored CWL (fCWL) test]
    Let $(\mathcal G, \mathcal X)$ be a graph and a cell complex constructed from a cellular lifting map that preserves the node set. The fCWL scheme is defined as follows:
\begin{itemize}[leftmargin=0.5cm]
    \itemsep0em
    \item {Initialization:} All cells are assigned the same initial color.
    \item {Color refinement:} At iteration $t+1$, the color of each non-node cell $\sigma$ is updated according to $c_\sigma^{t+1} = \text{HASH}(c_\sigma^t,\, c_{\mathcal{B}(\sigma)}^t, c_{\mathcal{C}(\sigma)}^t)$.
    The color of each node $i$ is updated according to $c_i^{t+1} = \text{HASH}(c_i^t, c_{\mathcal{C}(i)}^t, \,c_{\mathcal N(i)}^t)$.
    \item {Termination:} The process is repeated until the coloring stabilizes.  
    Two cell complexes are considered non-isomorphic if their color histograms differ.
\end{itemize}
\end{defi}

\begin{thm} \label{thm:expressivity_fcwl}
    fCWL is at least as expressive as WL and CWL.
\end{thm}
{We use the ideas from sCWL and fCWL tests to introduce cellular neural networks with the same guarantees and better scaling properties than CWN \citep{bodnar2021weisfeilercell}.}


\subsection{Neural Network Models}


We describe neural network architectures based on the CWL framework that perform message passing along the cell hierarchy via boundary, co-boundary, and adjacency relations.

\begin{defi}[CWNs]
\label{def:CWN}
    Following \citep[Section 4]{bodnar2021weisfeilercell}, CWNs aggregate messages along both upper adjacency and boundary relations (Theorem \ref{thm:expressivity_cwl}).
    For a cell $\sigma$, the updates are defined as:
    \begin{equation}
        \resizebox{\hsize}{!}{$\displaystyle
        \mathbf{m}_\uparrow(\sigma) = \bigoplus_{\substack{\tau \in \mathcal N_\uparrow(\sigma)\\ \delta \in \mathcal C(\sigma) \bigcap \mathcal C(\tau)}} 
        M_\uparrow\bigl( \mathbf{x}_\sigma, \mathbf{x}_\tau, \mathbf{x}_\delta \bigr),
        \;\; \mathbf{m}_{\mathcal B}(\sigma) = \bigoplus_{\tau \in \mathcal B(\sigma)} M_{\mathcal B}\bigl( \mathbf{x}_\sigma, \mathbf{x}_\tau \bigr),
        \;\; \mathbf{x}_\sigma = \bigoplus \left(\mathbf{x}_\sigma,\mathbf{m}_{\mathcal B}(\sigma), \mathbf{m}_\uparrow(\sigma) \right),
        $}
    \end{equation}  
    where  $\mathbf{x}_\sigma$ is the feature vector of cell $\sigma$. We write $\mathbf{m}_\uparrow(\sigma)$ for the aggregated message to cell $\sigma$ from all tuples formed by $\sigma$, one of its upper neighbors, and a parent they share. Similarly, $\mathbf{m}_{\mathcal{B}}(\sigma)$ denotes the aggregated message to cell $\sigma$ from all of its children.

\end{defi}




We now introduce a model that scales more efficiently.


\begin{defi}[Simplified CWNs (sCWN)]
\label{def:sCWN}
    Based on the restricted CWL update using only boundary and co-boundary messages (Theorem~\ref{thm:expressivity_scwl}), we define a simplified message-passing scheme:
    \begin{equation}
        \mathbf m_\mathcal{C}(\sigma) =  \bigoplus_{\tau \in \mathcal C(\sigma)} M_\mathcal{C}(\mathbf x_\tau), 
        \quad \mathbf m_\mathcal{B}(\sigma) = \bigoplus_{\tau \in \mathcal B(\sigma)} M_\mathcal{B}(\mathbf x_\tau),
        \quad \mathbf x_\sigma= \bigoplus (\mathbf x_\sigma, \mathbf m_\mathcal{C}(\sigma), \mathbf m_\mathcal{B}(\sigma)).
    \end{equation}
\end{defi}
Figure \ref{fig:aggregation_function} shows an example of the aggregation functions in Definition \ref{def:sCWN}.
This simplified variant reduces computational and memory requirements while retaining the expressive power of the full CWL update.
Messages are propagated only along boundary and co-boundary relations, making sCWN efficient for large complexes (see Proposition~\ref {prop:model_complexity}).




We also introduce a cell model that has a complexity between sCWN and CWN, but has better expressivity guarantees (Theorem~\ref{thm:expressivity_fcwl}, Proposition~\ref{prop:model_complexity}).
We use both the clique structure and the neighborhood structure from the graph.
\begin{defi}[Factored CWNs (fCWN)]
    \label{def:fCWN}
    fCWNs aggregate messages using the cell complex and graph structures:
    \begin{equation}
    \begin{gathered}
    \mathbf{m}_{\mathcal C}(\sigma) = \bigoplus_{\tau \in \mathcal C(\sigma)} M_{\mathcal C}\!\left(\mathbf{x}_\sigma, \mathbf{x}_\tau\right),
    \quad \mathbf{m}_{\mathcal B}(\sigma) = \bigoplus_{\tau \in \mathcal B(\sigma)} M_{\mathcal B}\!\left(\mathbf{x}_\sigma, \mathbf{x}_\tau\right), 
    \quad \mathbf{m}_{\mathcal N}(i) = \bigoplus_{j \in \mathcal N(i)} M_{\mathcal N}\!\left(\mathbf{x}_i, \mathbf{x}_j\right),\\
    \mathbf{x}_i = \bigoplus \left(\mathbf{x}_i, \mathbf{m}_{\mathcal C}(i), \mathbf{m}_{\mathcal N}(i)\right), \quad \mathbf{x}_\sigma = \bigoplus \left(\mathbf{x}_\sigma, \mathbf{m}_{\mathcal B}(\sigma), \mathbf{m}_{\mathcal C}(\sigma)\right).
    \end{gathered}
    \end{equation}
\end{defi}

\begin{wrapfigure}{r}{0.49\textwidth}
    \vspace{-10pt}
    \centering
    \includegraphics[width=0.42\textwidth]{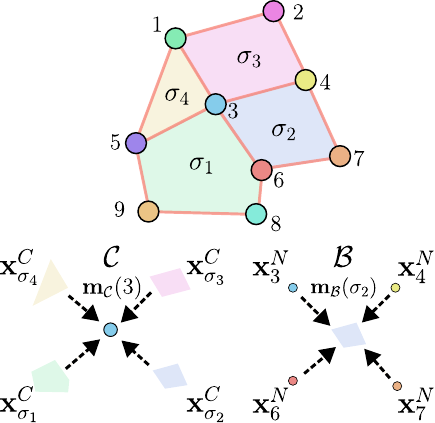}

\caption{
Simplified cellular message passing in sCWN as described in Definition \ref{def:sCWN}.
}
    \vspace{-20pt}
\label{fig:aggregation_function}
\end{wrapfigure}

This model has better memory and time complexity than CWN in practical cases and also better expressivity guarantees (Theorem~\ref{thm:expressivity_fcwl}).

Under certain constraints, we can provide expressivity guarantees for these models.
\begin{prop} \label{prop:CWN_expressive_CWL}
sCWN and CWN are at most as expressive as CWL.  
If they use injective aggregation, they are equally expressive as CWL. fCWN with injective aggregation is at least as expressive as CWL and WL.
\end{prop}

We can further scale the sCWN and fCWN models to large graphs while still exploiting higher-order topology. 
We propose to use clique-based cell complexes, where maximal cliques serve as higher-dimensional cells that compactly summarize multiple nodes and edges.


\subsection{CliqueWalk}

\begin{defi}[Maximal clique complex]
    \label{def:max_clique_complex}
    Given a graph $G = (\mathcal{V}, \mathcal{E})$, the \textbf{maximal clique complex} is defined as a cell complex where the $0$-cells correspond to the vertices $\mathcal{V}$ of $G$, and the higher-dimensional cells correspond to the maximal cliques of $G$. The set of non-$0$-cells (i.e., maximal cliques) is denoted by $\mathcal{X}$.
\end{defi}
An example of a maximal clique complex constructed from a graph is shown in Figure~\ref{fig:maximum_clique_complex}. 
If we impose closure under subset operations, the maximal clique complex becomes the \textit{clique complex}~\citep{kahle2009topology}, that is, the simplicial complex induced by including all subsets of each clique.

\begin{prop} \label{prop:model_complexity}
The time and memory complexities of the different CWN variants on {maximal clique complexes} are:
\begin{itemize}[leftmargin=0.5cm]
    \itemsep0em
    \item CWN has time and memory complexity $\mathcal{O}(n + \sum_{\sigma \in \mathcal{X}} |\sigma|^2)$.
    \item fCWN has time complexity $\mathcal{O}(|\mathcal E| + \sum_{\sigma \in \mathcal{X}} |\sigma| )$ and memory complexity $\mathcal{O}(n + \sum_{\sigma \in \mathcal{X}} |\sigma|)$.    
    \item sCWN has time complexity $\mathcal{O}(\sum_{\sigma \in \mathcal{X}} |\sigma|)$ and memory complexity $\mathcal{O}(n + |\mathcal{X}|)$.
\end{itemize}
Here, $n$ is the number of nodes, $\mathcal X$ is the set of maximal cliques, and $\mathcal E$ is  the set of edges. A table of all complexities can be found in the Appendix \ref{tab:mem_time_comp}.
\end{prop}

\begin{figure*}[t]
  \centering
  \includegraphics[width=\textwidth]{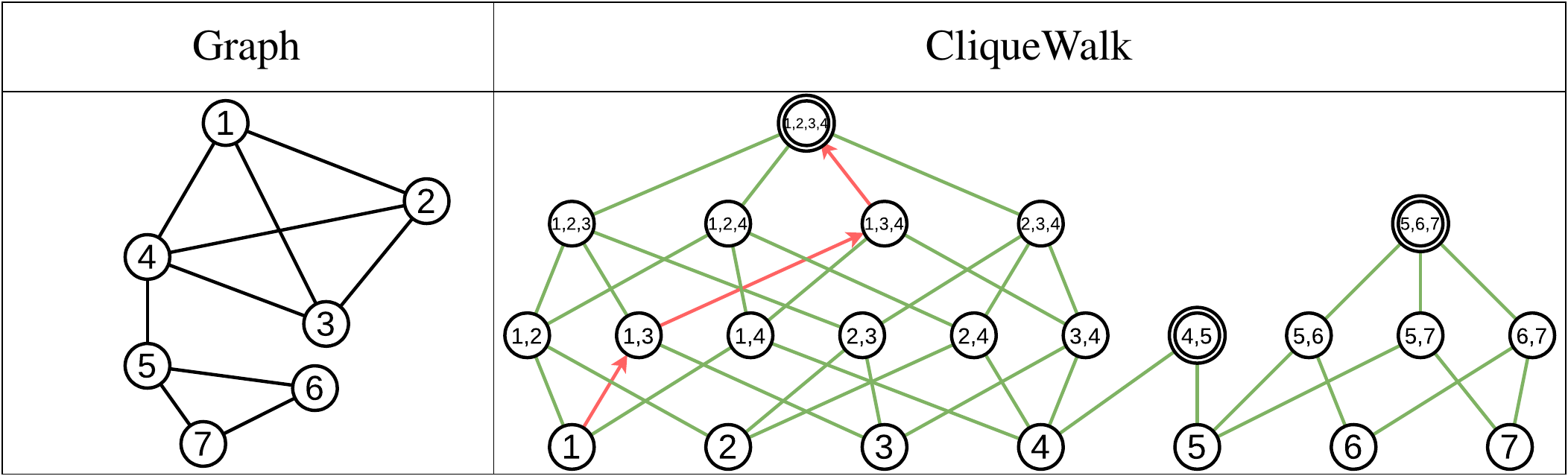}
  \caption{Illustration of a CliqueWalk starting at node~$1$. The walk grows a clique by repeatedly sampling nodes adjacent to all current clique members, successively adding nodes $3$, $4$, and $2$ to reach the maximal clique $\{1,2,3,4\}$.}
  \label{fig:cliquewalk}
\end{figure*}

\begin{rmq}
{
    These models can be simplified to reduce time and memory, for example, by using only incoming information during aggregation. These simplified versions keep the same theoretical expressivity but may capture less complex interactions between cells.}
\end{rmq}
\begin{rmq}
    We conjecture that CWL on maximal cliques is more expressive than WL. A discussion on this conjecture is present in Appendix \ref{app:conj}.
\end{rmq}



Identifying all maximal cliques in a graph is computationally infeasible, as the clique enumeration problem might have exponential runtime.


\begin{prop}[\citet{Moon1965}] \label{prop:maxclique_complex}
A graph with $n$ nodes can contain up to $3^{n/3}$ maximal cliques.
\end{prop}
To circumvent this challenge, we propose a biased random walk method for efficient clique sampling, which we refer to as CliqueWalk.
Our approach is inspired by maximal-clique enumeration algorithms \citep{Bron_1973, TOMITA200628, CAZALS2008564}.
The key idea is to grow cliques incrementally while maintaining an efficient lookup of candidate nodes that can extend the current clique, continuing until no further extension is possible.
The method is summarized in Algorithm~\ref{alg:cliquewalk} in Appendix \ref{sec:clique_sampling}, and is illustrated in Figure \ref{fig:cliquewalk}.
CliqueWalk enables us to sample a representative subset of cliques without exhaustively enumerating all of them.
A discussion comparing CliqueWalk with other clique sampling schemes can be found in Appendix~\ref{sec:clique_sampling}.




\begin{prop} \label{prop:maximal_clique}
If $\omega_{\text{max}}>\omega(G)$, each random walk generated by CliqueWalk produces a maximal clique of the graph, where $\omega(G)$ is the maximum clique size and $\omega_{\text{max}}$ is the maximum walk length.
\end{prop}

We denote our random walk method as $\text{CliqueWalk}(n_{\text{walk}}, \omega_{\text{max}})$, where $n_{\text{walk}}$ specifies the number of walks sampled per node and $\omega_{\text{max}}$ corresponds to maximum size of walks.

\begin{prop} \label{prop:complexity_cliquewalk}
The time complexity of $\text{CliqueWalk}(n_{\text{walk}}, \omega_{\text{max}})$ on a graph $G$ is $\mathcal{O} ( n \cdot n_{\text{walk}} \cdot d_{\text{max}}(G) \cdot \max(\omega(G), \omega_{\text{max}}) )$,
where $n$ is the number of nodes, $d_{\text{max}}(G)$ is the maximum node degree, and $\omega(G)$ is the size of the largest clique in $G$. (\ie linear in the number of edges at fixed $\omega_{max}$ and $n_{walk}$)
\end{prop}

 CliqueWalk is motivated by the prohibitive cost of enumerating all cliques in large graphs.
By sampling a sufficiently large number of cliques, we can approximate local clique structure, enabling efficient higher-order learning with performance comparable to, or better than, full enumeration.




\section{Experiments and Results}
 
In this section, we describe the datasets and experimental setups used for graph and node classification tasks.
We compare sCWN and fCWN with: GCN \citep{kipf2017semisupervisedclassificationgraphconvolutional}, GIN \citep{xu2018how}, SAGEConv \citep{hamilton2018inductiverepresentationlearninglarge}, SGC \citep{wu2019simplifying}, HGNN \citep{feng2019hypergraphneuralnetworks}, SCCN \citep{yang22a}, CIN, CWN \citep{bodnar2021weisfeilercell}, PPGN \citep{Maron2019} and G2N2 \citep{Piquenot2024}.


\subsection{Datasets}

\textbf{Graph classification datasets.} 
We perform experiments on two social network datasets (IMDB-BINARY, IMDB-MULTI)~\citep{Yanardag_2015} and four molecular or biological datasets (MUTAG, PROTEINS, NCI1, NCI109) from the TUDataset~\citep{morris2020tudatasetcollectionbenchmarkdatasets}.

\textbf{Node classification datasets.} 
We evaluate our models on two topological datasets ({contact-primary-school} and {contact-high-school})~\citep{chodrow2021hypergraph, Mastrandrea-2015-contact}, three citation networks (Citeseer, Cora, and PubMed)~\citep{Sen_Namata_2008, Namata2012QuerydrivenAS}, and purchase networks like Amazon Photo network~\citep{mcauley2015amazon, shchur2019pitfallsgraphneuralnetwork} and {OGBN-Products (OGBN-P)~\citep{Bhatia16}}.
In addition, we propose a new synthetic dataset, the \textit{stochastic clique model} (SCM), derived from the stochastic block model (SBM). SCM is a special case of the SBM~\citep{HOLLAND1983109} with inward probability set to $1$.
Graphs are generated by assembling cliques, each assigned a label inherited by its nodes.
Node features are sampled from a Gaussian distribution with label-dependent means and fixed diagonal variance.
Topological noise is introduced by randomly connecting nodes across cliques, so the task can be viewed as label denoising.

\textbf{Synthetic cliques.} 
To compare the inference time and memory footprint of clique-based methods, we also construct a synthetic dataset of isolated cliques. 
This dataset allows us to systematically evaluate the computational scaling of CWN models with respect to clique size.

\subsection{Experimental Setup}

\textbf{Experiments.}
For graph classification, we follow the protocol of \cite{xu2018how}, using $10$-fold cross-validation, reporting the best mean accuracy across folds at each epoch and selecting the epoch with the highest mean accuracy for final evaluation. 
For PPGN \cite{Maron2019}, G2N2 \cite{Piquenot2024}, and CIN \cite{bodnar2021weisfeilercell}, we report the results directly from the original papers as they follow the same protocol.
For node classification, we hold out 20\% of nodes as a final test set.
The remaining nodes are split into 60\%/20\%/20\% for training, validation, and internal testing during hyperparameter tuning. 
We select the checkpoint with the highest validation accuracy and report its performance on the final test set.
For OGBN-Products, we use the public splits.


\textbf{Implementation details.}
{In all experiments, we use a shared architecture and vary only the convolutional module corresponding to the method under evaluation. We run SCCN on the standard lifted simplicial complex up to triangles  \cite{yang22a}. For our decomposition-based methods, we evaluate CWN, sCWNc (\textit{i.e.}, sCWN applied to cliques), fCWN, and HGNN on complexes obtained via the CliqueWalk sampling procedure with $8$ walks per node and initialize clique features using clique size.  We choose $8$ walks as it provides a good trade-off between accuracy and computational cost across datasets. The sampled cliques are generated once and kept fixed throughout training (\textit{i.e.}, no resampling). To ensure a fair comparison with CIN \cite{bodnar2021weisfeilercell}, we follow its original setup: we evaluate sCWNr (sCWN on rings) using the same topological lifting (node–edge–cycle) and readout scheme described in \cite{bodnar2021weisfeilercell}. 
For graph classification, all models use five layers (including the input convolution) and a hidden dimension of $64$, while grid search is limited to dropout $\{0,0.5\}$, batch size $\{32,128\}$, and {with or without BatchNorm}. 
To ensure a fair comparison of memory usage across models, all reported peak memory values are measured under a unified setting with the same architecture and hyperparameters (three layers, hidden dimension $64$, and batch size $32$). This emphasizes that the observed differences stem from the underlying model design rather than tuning choices.}

For node classification, except OGBN-Products, we perform a grid search over learning rate $\{10^{-2},10^{-3}\}$, number of layers $\{2,4\}$, hidden dimension $\{32,64\}$, dropout $\{0,0.2,0.5\}$ and {with or without BatchNorm for all models}. 
For \textit{contact-school} datasets, we also include GraphNorm\footnote{The estimation of the statistics with BatchNorm on small datasets degrades model performance.} \citep{Cai2020GraphNormAP}.
Models are trained for $200$ epochs on standard datasets and $500$ epochs on topological ones\footnote{GNNs converge more slowly on topological datasets, hence the larger number of epochs.}, with each grid search repeated five times using different random seeds.
Final evaluation is based on $20$ independent runs with new seeds. 
For OGBN-Products, we tune hyperparameters for our clique models, use typical hyperparameters for GNN baselines, and train all models for $1000$ epochs and average over $5$ runs (see Appendix~\ref{app:datasets} for further details).

\subsection{Results and Discussion}

\begin{table*}[t]
\centering
\definecolor{dense}{HTML}{9b59b6}
\caption{Graph classification accuracy (\%) and peak GPU memory (MB).
  Best accuracy in \textbf{bold}, second best \underline{underlined}.
  {\color{GNN}$\blacklozenge$}~GNNs,
  {\color{HNN}$\spadesuit$}~hypergraph NNs,
  {\color{CWN}$\maltese$}~CW-networks, {\color{dense}$\clubsuit$}~dense higher-order models, {\color{ours}$\bigstar$} fCWN, sCWNc, sCWNr (ours).
  OOM = out of memory.}
\label{tab:graph_classification}
\resizebox{\textwidth}{!}{%
\begin{tabular}{lcccccccccccc}
\toprule
\multirow{2}{0em}{\textbf{Model}} & \multicolumn{2}{c}{\textbf{IMDB-B}} & \multicolumn{2}{c}{\textbf{IMDB-M}} & \multicolumn{2}{c}{\textbf{MUTAG}} & \multicolumn{2}{c}{\textbf{NCI1}} & \multicolumn{2}{c}{\textbf{NCI109}} & \multicolumn{2}{c}{\textbf{PROTEINS}} \\
\cmidrule(lr){2-3} \cmidrule(lr){4-5} \cmidrule(lr){6-7} \cmidrule(lr){8-9} \cmidrule(lr){10-11} \cmidrule(lr){12-13}
 & Acc. & Mem. & Acc. & Mem. & Acc. & Mem. & Acc. & Mem. & Acc. & Mem. & Acc. & Mem. \\
\midrule
{\color{GNN}$\blacklozenge$} GCN & $74.3_{\pm 4.6}$ & $19$ & $52.4_{\pm 4.1}$ & $18$ & $84.1_{\pm 8.8}$ & $18$ & $80.4_{\pm 1.8}$ & $19$ & $76.9_{\pm 1.7}$ & $19$ & $77.0_{\pm 5.1}$ & $20$ \\
{\color{GNN}$\blacklozenge$} GAT & $74.8_{\pm 3.0}$ & $26$ & $51.6_{\pm 3.7}$ & $25$ & $84.6_{\pm 8.6}$ & $19$ & $79.6_{\pm 3.1}$ & $21$ & $73.8_{\pm 1.3}$ & $21$ & $76.5_{\pm 3.2}$ & $27$ \\
{\color{GNN}$\blacklozenge$} GIN & $72.1_{\pm 3.8}$ & $19$ & $49.7_{\pm 3.4}$ & $18$ & $89.4_{\pm 7.8}$ & $18$ & $80.8_{\pm 2.1}$ & $20$ & $74.8_{\pm 2.4}$ & $19$ & $75.8_{\pm 3.4}$ & $21$ \\
{\color{GNN}$\blacklozenge$} SAGEConv & $74.3_{\pm 4.1}$ & $19$ & $\underline{52.9_{\pm 4.0}}$ & $18$ & $84.6_{\pm 9.5}$ & $18$ & $81.5_{\pm 1.8}$ & $19$ & $78.0_{\pm 1.5}$ & $19$ & $76.3_{\pm 4.5}$ & $20$ \\
\midrule
\midrule
{\color{HNN}$\spadesuit$} HGNN & $75.5_{\pm 4.3}$ & $19$ & $52.3_{\pm 4.8}$ & $18$ & $86.2_{\pm 8.2}$ & $18$ & $79.2_{\pm 3.1}$ & $20$ & $76.2_{\pm 1.9}$ & $20$ & $76.5_{\pm 3.9}$ & $21$ \\
{\color{CWN}$\maltese$} CWN & $66.0_{\pm 7.8}$ & $40$ & $50.5_{\pm 3.4}$ & $35$ & $85.1_{\pm 7.3}$ & $22$ & $63.7_{\pm 1.9}$ & $26$ & $63.1_{\pm 2.0}$ & $26$ & $77.0_{\pm 3.4}$ & $40$ \\
{\color{CWN}$\maltese$} CIN & $\underline{75.6_{\pm 3.7}}$ & $454$ & $52.7_{\pm 3.1}$ & $517$ & $\mathbf{92.7_{\pm 6.1}}$ & $31$ & $\underline{83.6_{\pm 1.4}}$ & $42$ & $\underline{84.0_{\pm 1.6}}$ & $41$ & $77.0_{\pm 4.3}$ & $137$ \\
{\color{dense}$\clubsuit$} PPGN & $73.0_{\pm 5.8}$ & $2193$ & $50.5_{\pm 3.6}$ & $950$ & $90.6_{\pm 8.7}$ & $110$ & $83.2_{\pm 1.1}$ & $1467$ & $82.2_{\pm 1.4}$ & $1467$ & $77.2_{\pm 4.7}$ & $45277$ \\
{\color{dense}$\clubsuit$} G2N2 & $\mathbf{76.8_{\pm 2.8}}$ & $846$ & $\mathbf{54.0_{\pm 2.9}}$ & $400$ & $\underline{92.5_{\pm 5.5}}$ & $189$ & $82.8_{\pm 0.9}$ & $858$ & \text{--} & $829$ & $\mathbf{80.1_{\pm 3.7}}$ & $16188$ \\
\midrule
{\color{ours}$\bigstar$} fCWN & $71.9_{\pm 4.1}$ & $19$ & $52.8_{\pm 2.6}$ & $19$ & $85.1_{\pm 8.1}$ & $20$ & $79.2_{\pm 2.4}$ & $22$ & $62.3_{\pm 4.5}$ & $22$ & $75.9_{\pm 3.3}$ & $24$ \\
{\color{ours}$\bigstar$} sCWNc & $75.0_{\pm 4.5}$ & $19$ & $52.3_{\pm 4.2}$ & $19$ & $85.7_{\pm 8.2}$ & $20$ & $66.3_{\pm 8.9}$ & $22$ & $64.1_{\pm 2.8}$ & $22$ & $\underline{77.5_{\pm 3.5}}$ & $25$ \\
{\color{ours}$\bigstar$} sCWNr & $74.6_{\pm 3.5}$ & $139$ & $51.3_{\pm 3.6}$ & $178$ & $91.7_{\pm 6.2}$ & $22$ & $\mathbf{83.8_{\pm 1.6}}$ & $26$ & $\mathbf{84.1_{\pm 1.2}}$ & $26$ & $75.9_{\pm 4.9}$ & $38$ \\
\bottomrule
\end{tabular}
}%
\end{table*}

\begin{table*}[t]
\centering
\caption{Node classification accuracy (\%) with standard deviation. 
Best results are in \textbf{bold}, second best are \underline{underlined}.
HighSchool = contact-high-school, PrimarySchool = contact-primary-school. {\color{GNN}$\blacklozenge$} GNNs, {\color{SNN}$\clubsuit$} simplicial neural networks, {\color{HNN}$\spadesuit$} hypergraph neural networks, {\color{CWN}$\maltese$} CWN, and {\color{ours}$\bigstar$} fCWN, sCWNc (ours).} 
\label{tab:node_classification}
\resizebox{\textwidth}{!}{
\begin{tabular}{lcccccccc} 
\toprule
\textbf{Model} & \textbf{Citeseer} & \textbf{Cora} & \textbf{Photo} & \textbf{PubMed} & \textbf{HighSchool} & \textbf{PrimarySchool} & \textbf{SCM} & {\textbf{OGBN-P}}\\
\midrule
{\color{GNN}$\blacklozenge$} GCN        & $\mathbf{73.7_{\pm 0.76}}$ & $\mathbf{88.7_{\pm 0.61}}$ & $93.9_{\pm 0.27}$ & $88.3_{\pm 0.33}$ & ${\underline{98.2_{\pm 2.6}}}$ & ${88.9_{\pm 3.1}}$ & OOM & $76.3_{\pm 0.4}$ \\
{\color{GNN}$\blacklozenge$} GAT        & $72.2_{\pm 1.3}$ & $87.5_{\pm 1.2}$ & $93.7_{\pm 0.26}$ & $87.2_{\pm 0.33}$ & ${19.1_{\pm 7.3}}$ & ${13.9_{\pm 7.8}}$ & OOM & OOM\\
{\color{GNN}$\blacklozenge$} GIN        & $69.3_{\pm 1.1}$ & $86.2_{\pm 0.62}$ & $88.0_{\pm 2.2}$ & $86.7_{\pm 0.42}$ & ${94.5_{\pm 3.5}}$ & ${85.9_{\pm 4.6}}$ & OOM & $76.4_{\pm0.4}$ \\
{\color{GNN}$\blacklozenge$} SAGEConv  & $72.4_{\pm 1.2}$ & $\mathbf{88.7_{\pm 0.99}}$ & $95.0_{\pm 0.29}$ & $\underline{89.5_{\pm 0.6}}$ & ${14.6_{\pm 4.2}}$ & ${6.53_{\pm 4.5}}$ & OOM & $\underline{78.5_{\pm 0.1}}$\\
{\color{GNN}$\blacklozenge$} SGC  & $\mathbf{73.7_{\pm 0.74}}$ & $88.4_{\pm0.86}$ &  $89.8_{\pm0.39}$ &  $89.2_{\pm0.21}$& $6.3_{\pm 4.1}$ & $3.57_{\pm 3.0}$ & $65.6_{\pm0.01}$ & $76.1_{\pm 0.07}$\\
\midrule
\midrule
{\color{SNN}$\clubsuit$} SCCN       & $46.4_{\pm 1.4}$ & $64.4_{\pm 1.9}$ & $64.8_{\pm 2.6}$ & $73.4_{\pm 0.7}$ & $93.0_{\pm 2.5}$ & ${74.1_{\pm 3.7}}$ & OOM & OOM\\[0.2em]
{\color{HNN}$\spadesuit$} HGNN       & $\underline{72.9_{\pm 1.1}}$ & $\underline{88.5_{\pm 0.9}}$ & $94.2_{\pm 0.5}$ & $88.5_{\pm 0.39}$ & ${95.4_{\pm 3.8}}$ & ${80.4_{\pm 5.3}}$ & $\underline{68.1_{\pm 0.3}}$ & $78.3_{\pm 0.4}$\\[0.2em]
{\color{CWN}$\maltese$} CWN        & $72.0_{\pm1.6}$ & $81.1_{\pm 1.0}$ & $94.7_{\pm0.37}$ & $89.3_{\pm 0.35}$ & $94.6_{\pm 2.2}$ & ${\mathbf{90.7_{\pm1.9}}}$ & OOM & OOM
\\
\midrule
{\color{ours}$\bigstar$} fCWN      &  $72.5_{\pm 1.4}$ & $88.1_{ \pm 0.79}$ & $\underline{95.1_{ \pm 0.35}}$ & $89.4_{ \pm 0.31}$ & ${\mathbf{99.5_{ \pm 0.9}}}$ & ${\underline{89.5_{ \pm 2.3}}}$ & OOM & $\mathbf{80.1_{\pm 0.2}}$\\

{\color{ours}$\bigstar$} sCWNc       & $\underline{72.9_{\pm 1.3}}$ & $87.3_{\pm 0.87}$ & $\mathbf{95.3_{\pm 0.39}}$ & $\mathbf{89.7_{\pm 0.35}}$ & $96.0_{\pm 2.4}$ & $86.4_{\pm 4.4}$ & $\mathbf{77.7_{\pm0.05}}$ & $76.6_{\pm 0.2}$\\
\bottomrule
\end{tabular}
}
\end{table*}


\textbf{Graph classification.} {Table~\ref{tab:graph_classification} summarizes the results of the graph classification task. On social network datasets such as IMDB-B and IMDB-M, topological models achieve strong performance, consistent with prior work~\citep{bodnar2021weisfeilercell}. In particular, higher-order methods such as CIN and G2N2 often yield the best accuracies. However, this improved performance comes at a substantial computational cost: these models exhibit significantly higher peak memory usage, even on relatively small datasets.
In contrast, our methods (sCWN and fCWN) achieve competitive performance while maintaining a memory footprint comparable to standard GNNs. This highlights a key trade-off: while dense higher-order models can perform better, their poor scalability limits their applicability to large-scale graphs.
On molecular datasets, clique-based methods tend to underperform. This can be explained by the structural properties of these graphs: the average maximal clique size is close to 2 (see Table \ref{tab:dataset_statistics}), making clique sampling effectively similar to edge sampling and limiting the benefit of higher-order clique-based features.}

\begin{wrapfigure}{r}{0.48\textwidth}
    \includegraphics[width=\linewidth]{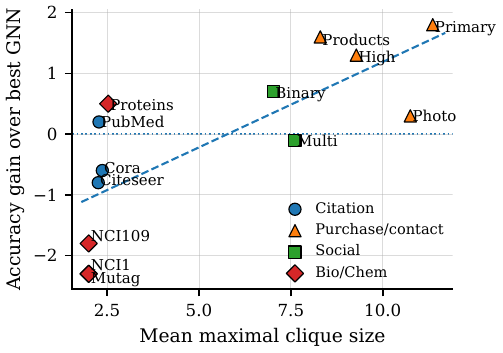}
    \caption{Accuracy gain over the best GNN baseline as a function of mean clique size. Larger cliques tend to favor clique-based models.}
    \vspace{-10pt}
    \label{fig:acc_clique_cor}
\end{wrapfigure}

\textbf{Node classification.}
Table~\ref{tab:node_classification} reports node classification results. The SCM dataset, with approximately $6$M nodes and $276$M edges, is substantially larger than standard benchmarks; for this dataset, we use only one CliqueWalk random walk. Additional dataset statistics are provided in Table~\ref{tab:dataset_statistics} (Appendix~\ref{app:datasets}). On topological datasets such as contact-high-school and contact-primary-school, topological models are competitive, while classical GNNs with GraphNorm perform similarly, with fCWN slightly better. On citation benchmarks (Citeseer, Cora), differences are small, indicating no clear advantage for topological methods. 
On OGBN-Products, fCWN and HGNN with CliqueWalk are competitive with GNNs that only use the edges.
In contrast to other methods that run out of memory (OOM), sCWN and fCWN scale efficiently.

\begin{figure*}
    \centering
    \begin{subfigure}[b]{0.48\textwidth}
        \centering
        \includegraphics[width=\textwidth]{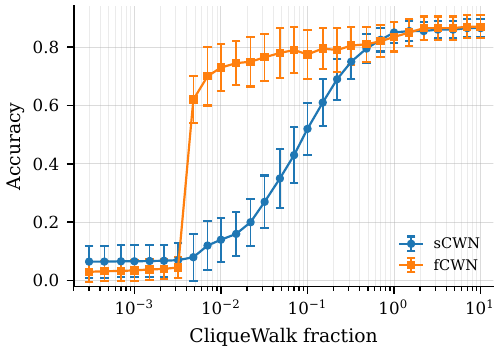}
        \caption{Accuracy of sCWN and fCWN on contact-primary-school vs. CliqueWalk sampling rates. }
        \label{fig:walk_effect}
    \end{subfigure}
    \hspace{0.02\textwidth}
    \begin{subfigure}[b]{0.48\textwidth}
        \centering
        \includegraphics[width=\textwidth]{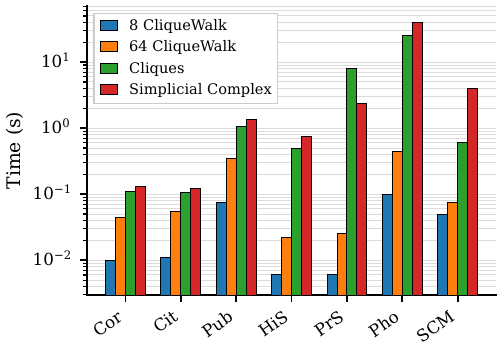}
        \caption{Computation time of different lifting strategies measured on an NVIDIA RTX~3090 GPU.}
        \label{fig:lifting_time}
    \end{subfigure}
    \caption{Sensitivity analysis of CliqueWalk. (a) Accuracy versus sampled walks. (b) Runtime comparison between CliqueWalk and exact lifting methods. Cor = Cora, Cit = Citeseer, Pub = PubMed, HiS = contact-high-school, PrS = contact-primary-school, Pho = Photo.}
    \label{fig:sensitivity_analysis}
\end{figure*}




\begin{rmq}
Across graph and node classification, topological models perform better on datasets with larger cliques. Table~\ref{tab:dataset_statistics} in Appendix~\ref{app:datasets} reports average clique sizes, and Figure~\ref{fig:acc_clique_cor} reveals a clear correlation between larger cliques and improved topological performance. 
\end{rmq}




\subsection{Sensitivity Analysis and Ablation Study}
\label{subsec:ablation}


\textbf{Sampling effect for CliqueWalk.} 
We evaluate the impact of clique density on performance by scaling the sampling rate from nearly zero to 10 CliqueWalks per node (Figure \ref{fig:walk_effect}) in contact-primary-school. On this topological dataset, we observe that predictive accuracy scales directly with the number of sampled cliques; higher sampling rates capture richer structural information, leading to better predictions. We observe a different effect on classical datasets like Photo, as shown in Appendix~\ref{app:ablations}.

\textbf{CliqueWalk compute time.}
We compare CliqueWalk with $8$ and $64$ walks to exact clique enumeration and triangle-based simplicial lifting (Figure~\ref{fig:lifting_time}).
Across all datasets, CliqueWalk achieves substantially lower runtimes; even with 64 walks per node, it is close to an order of magnitude faster while maintaining competitive accuracy.
This demonstrates that CliqueWalk is an efficient and scalable alternative to exact methods.

\begin{wraptable}{r}{0.45\textwidth}
\centering
\vspace{-10pt}
\caption{Ablation sampling CliqueWalk.}
\label{tab:ablation_sampling}
\resizebox{\linewidth}{!}{
\begin{tabular}{lcc}
\toprule
\textbf{Strategy} & \textbf{PrimarySchool} & \textbf{Photo} \\
\midrule
Re-sampling & $87.1_{\pm3.8}$ & $95.4_{\pm0.44}$ \\
No re-sampling & $86.4_{\pm 4.4}$ & $95.3_{\pm 0.39}$ \\
\bottomrule
\end{tabular}
}
\vspace{-10pt}
\end{wraptable}




\textbf{Ablation study, resampling in CliqueWalk.}
Table \ref{tab:ablation_sampling} compares the performance when using $8$-walk CliqueWalk with or without re-sampling at each training epoch on the \textit{contact-primary-school} and \textit{Photo} datasets. 
We observe that the results are slightly better across both datasets when re-sampling, while it introduces a slight increase in runtime (see PrS and Pho in Figure \ref{fig:lifting_time}).
This suggests that using resampling can be a more practical way to trade better generalization against computational cost.


\subsection{Limitations}

While our work establishes a scalable framework for clique-based higher-order learning, it has some limitations.
First, we restrict our evaluation to node and graph classification tasks; extending the approach to other settings, such as hyperedge prediction, link prediction, or generative modeling, remains an open direction.
Second, our method does not explicitly expand the receptive field of nodes, and thus may not fully capture long-range dependencies compared to approaches that incorporate multi-hop information.
Finally, we focus on clique-based sampling strategies, whereas exploring alternative lifting procedures or hybrid strategies could further improve performance and generalization. 
Addressing these limitations offers promising avenues for future research.








\section{Conclusion}
We introduced the maximal clique complex as a simplified higher-order structure that connects clique-based representations to the CWL test, and showed that a sCWN operating on this complex achieves CWL-level expressivity while remaining computationally efficient.
To address scalability, we proposed CliqueWalk, a biased random walk algorithm that samples cliques efficiently and scales quasi-linearly with the number of nodes.
Together, these contributions enable the design of clique-based neural architectures that are both expressive and scalable.
Extensive experiments on node and graph classification benchmarks demonstrate that our models achieve competitive performance compared to GNNs and other higher-order approaches, while maintaining substantially lower memory and runtime requirements.
This work establishes random walk clique-based lifting as a practical path toward scalable higher-order graph learning. It opens the door for future research on efficient sampling strategies and domain-specific applications.

\ifarxiv 
    \section*{Acknowledgments}
    The authors acknowledge the French National Research Agency (ANR) for its financial support of the JCJC project DeSNAP ANR-24-CE23-1895-01. This work is also supported by Hi! PARIS and ANR/France 2030 program (Grants ANR-23-IACL-0005 and ANR-23-CMAS-0033).
\else
\fi

\bibliographystyle{plainnat}
\bibliography{references}

\newpage
\appendix

\doparttoc 
\faketableofcontents 

\addcontentsline{toc}{section}{Appendix} 
\part{{\Large Appendix}} 
\parttoc 
\newpage

\section{Weisfeiler-Leman Graph Isomorphism Test}
\label{app:WL_Test}

\begin{defi}
Let $A(\cdot)$ and $B(\cdot)$ be graph hashing functions.
We say that $A$ is more expressive than $B$ if, for any pair of graphs $G$ and $G'$, if the following condition holds:
\begin{equation}
    B(G) \neq B(G') \;\; \implies \;\; A(G)  \neq A(G').
\end{equation}
\end{defi}
Intuitively, a more expressive hashing can distinguish a wider range of non-isomorphic graphs.

A classical and widely used technique for graph isomorphism testing is the \emph{Weisfeiler–Leman (WL) test} \citep{weisfeiler1968reduction}, which is based on iterative color refinement:
\begin{defi}
The WL test constructs, in an iterative manner, a mapping $c$ from the nodes of a graph to a finite set of colors as follows:
\begin{itemize}
    \setlength\itemsep{0pt}
    \item {Initialization:} All nodes are assigned the same initial color.
    \item {Color refinement:} At iteration $t+1$, the color of each node $i$ is updated according to $c_i^{t+1} = \text{HASH}\!\left(c_i^t, \{\!\{c_j^t : j \sim i\}\!\}\right)$, where $j \sim i$ denotes that node $j$ is adjacent to node $i$, and $\text{HASH}$ is an injective function.
    \item {Termination:} The process continues until the coloring no longer changes.  
    Two graphs are considered non-isomorphic if their color histograms differ; otherwise, the test does not provide a conclusive answer.
\end{itemize}
\end{defi}
The WL test provides an efficient heuristic for the graph isomorphism problem~\citep{Huang_2021}.

\section{Proofs}
\label{app:proofs}


\subsection{Proof of Theorem \ref{thm:expressivity_scwl}}

First, we introduce the same notations, definitions, and propositions as in \citep{bodnar2021weisfeilercell} to manipulate cellular coloring.

\begin{defi}
    A \textbf{cellular coloring} is a function c that maps a cell complex $X$ and one of its cells $\sigma$ to a finite set (color set). We denote this color as $c_\sigma^X$.
\end{defi}

\begin{defi}
    Let $X,Y$ be two cell complexes and $c$ a coloring. We say that $X$ and $Y$ are $c$-similar, denote as $c^X = c^Y$ if $\{\!\{c_\sigma^X, \quad \sigma \in X \}\!\} = \{\!\{c_\tau^Y, \quad \tau \in Y\}\!\}$. Otherwise, we have $c^X \neq c^Y$.
\end{defi}

\begin{defi}
    A coloring $c$ is said to \textbf{refine} another coloring $d$, denoted $c \subseteq d$, if for all cell complexes $X, Y$ and all $\sigma \in X, \tau \in Y$, we have:
\[
c_\sigma^X = c_\tau^Y \;\;\implies\;\; d_\sigma^X = d_\tau^Y.
\]
If both $c \subseteq d$ and $d \subseteq c$, then the two colorings are said to be \textbf{equivalent}, denoted $c \equiv d$.
\end{defi}

\begin{prop} \label{prop:include_col_hash}
    Let $X, Y$ be cell complexes with $A \subseteq X$ and $B \subseteq Y$. Consider two colorings $c, d$ such that $c \subseteq d$.
\[
\{\!\{c_\sigma^X, \quad \sigma \in A \}\!\} \;=\; \{\!\{ c_\tau^Y , \quad \tau \in B \}\!\}
 \implies
\{\!\{d_\sigma^X , \quad \sigma \in A \}\!\} \; =\; \{\!\{ d_\tau^Y , \quad \tau \in B \}\!\}.
\]
\end{prop}

\begin{proof}
    Suppose that $\{\!\{c_\sigma^X ,\quad \sigma \in A \}\!\} \;=\; \{\!\{ c_\tau^Y ,\quad \tau \in B \}\!\}$. It means that there exists a bijection $f : A \to B$ such that for all $\sigma \in A$, $c_\sigma^X = c_{f(\sigma)}^Y$.
    \\ As $c \subseteq d$, 
    $d_\sigma^X = d_{f(\sigma)}^Y$ \ie $\{\!\{d_\sigma^X ,\quad \sigma \in A \}\!\} \; =\; \{\!\{ d_\tau^Y ,\quad \tau \in B \}\!\}$.
\end{proof}

\begin{cor}
    If $c \subseteq d$, then for all cell complexes $X, Y$,
\[
c^X = c^Y \;\;\implies\;\; d^X = d^Y.
\]

All non-distinguished cell complexes by $c$ are not distinguished by $d$. In other words, $c$ is a more powerful isomorphic test than $d$.
\end{cor}

\textbf{Proof of Theorem \ref{thm:expressivity_scwl}.} We show that CWL with coloring HASH$(c_\sigma^t, c_\mathcal B^t, c_\mathcal C^t)$ is as powerful as HASH$(c_\sigma^t, c_\mathcal B^t, c_\mathcal \uparrow^t)$ .
$a^t$ denotes   the colouring at step $t$ using CWL with HASH$(c_\sigma^t, c_\mathcal B^t, c_\mathcal \uparrow^t)$ and $b^t$ the one using HASH$(c_\sigma^t, c_\mathcal B^t, c_\mathcal C^t)$.
We know that the coloring $a^t$ is as powerful as the original CWL (Theorem 7, in \cite{bodnar2021weisfeilercell}). Since $b^t$ uses a subset of the CWL coloring relationships, it can be shown by induction that it is less powerful than the original CWL. Therefore, we have $a \subseteq b$.

We show that $b \subseteq a$.

We show by induction that $b^{2t} \subseteq a^t$ for all $t \in \mathbb N$.

\textit{Base case.} $b^0 \subseteq a^0$ as they follow the same color initialization scheme.

\textit{Inductive step.}  Assume $b^{2t} \subseteq a^t$. We prove that $b^{2t+2} \subseteq a^{t+1}$.

let $(\sigma_1,\sigma_2) \in X \times Y$ such that $b_{\sigma_1}^{2t+2} = b_{\sigma_2}^{2t+2}$. By construction, $$b_{\sigma_1}^{2t+1} = b_{\sigma_2}^{2t+1}, \quad b_{\mathcal B}^{2t+1}(\sigma_1) = b_{\mathcal B}^{2t+1}(\sigma_2), \quad b_{\mathcal C}^{2t+1}(\sigma_1) = b_{\mathcal C}^{2t+1}(\sigma_2),$$
as $b_{\mathcal C}^{2t+1}(\sigma_1) = b_{\mathcal C}^{2t+1}(\sigma_2)$, there exist a bijective map $f : \mathcal C(\sigma_1) \to \mathcal C(\sigma_2)$ that preserve the $b^{2t+1}$ coloring ie $b_{\tau}^{2t+1} = b_{f(\tau)}^{2t+1}$ for $\tau \in \mathcal C(\sigma_1)$.

As $b_{\tau}^{2t+1} = b_{f(\tau)}^{2t+1}$, we have  $b_{\mathcal B}^{2t}(\tau) = b_{\mathcal B}^{2t}(f(\tau))$, \textit{i.e.},
$$\{\!\{b_\gamma^{2t}, \quad \gamma \in \mathcal B(\tau) \}\!\} = \{\!\{b_\gamma^{2t}, \quad \gamma \in \mathcal B(f(\tau)) \}\!\}.$$
We can add the color of $\tau$ on both sides, the multisets would still stay equal:
 $$\{\!\{(b_\gamma^{2t}, b_\tau^{2t}), \quad \gamma \in \mathcal B(\tau) \}\!\} = \{\!\{(b_\gamma^{2t}, b_\tau^{2t}), \quad \gamma \in \mathcal B(f(\tau)) \}\!\}.$$

As this is true for all $\tau$ in $\mathcal C(\sigma_1)$, we can take the union: 
$$\underset{\tau \in \mathcal C(\sigma_1)}{\bigcup} \{\!\{(b_\gamma^{2t}, b_\tau^{2t}), \quad \gamma \in \mathcal B(\tau) \}\!\}  = \underset{\tau \in \mathcal C(\sigma_1)}{\bigcup} \{\!\{(b_\gamma^{2t}, b_\tau^{2t}), \quad \gamma \in \mathcal B(f(\tau)) \}\!\},$$
\textit{i.e.},
$$ \{\!\{(b_\gamma^{2t}, b_\tau^{2t}), \quad  \tau \in \mathcal C(\sigma_1), \gamma \in \mathcal B(\tau)\}\!\}  = \{\!\{(b_\gamma^{2t}, b_\tau^{2t}), \quad \tau \in \mathcal C(\sigma_1),  \gamma \in \mathcal B(f(\tau)) \}\!\} ,$$
as $b_\tau^{2t} = b_{f(\tau)}^{2t}$ and $f$ is bijective, the right term can be simplified:
\begin{align*}
    \{\!\{(b_\gamma^{2t}, b_\tau^{2t}), \quad \tau \in \mathcal C(\sigma_1),  \gamma \in \mathcal B(f(\tau)) \}\!\} &= \{\!\{(b_\gamma^{2t}, b_{f(\tau)}^{2t}), \quad \tau \in \mathcal C(\sigma_1),  \gamma \in \mathcal B(f(\tau)) \}\!\} \\ &= \{\!\{(b_\gamma^{2t}, b_{\delta}^{2t}), \quad \delta \in \mathcal C(\sigma_2),  \gamma \in \mathcal B(\delta) \}\!\},
\end{align*}
\textit{i.e.},
$$ \{\!\{(b_\gamma^{2t}, b_\tau^{2t}), \quad  \tau \in \mathcal C(\sigma_1), \gamma \in \mathcal B(\tau)\}\!\}  = \{\!\{(b_\gamma^{2t}, b_{\delta}^{2t}), \quad \delta \in \mathcal C(\sigma_2),  \gamma \in \mathcal B(\delta) \}\!\}.$$

Thus $b_\uparrow^{2t}(\sigma_1) = b_\uparrow^{2t}(\sigma_2)$. Using  the induction hypothesis $b^{2t} \subseteq a^t$ with proposition \ref{prop:include_col_hash},  we have $$a_{\sigma_1}^t = a_{\sigma_2}^t \quad   a_\uparrow^{t}(\sigma_1) = a_\uparrow^{t}(\sigma_2) \quad a_\mathcal B^{t}(\sigma_1) = a_\mathcal B^{t}(\sigma_2) \quad a_\mathcal C^{t}(\sigma_1) = a_\mathcal C^{t}(\sigma_2),$$ 
\textit{i.e.},
$$a_{\sigma_1}^{t+1} = a_{\sigma_2}^{t+1}.$$

From our induction $b^{2t} \subseteq a^t$ for all $t \in \mathbb{N}$, hence $b \subseteq a$.\qed

\subsection{Proof of Theorem \ref{thm:expressivity_fcwl} and Proposition \ref{prop:CWN_expressive_CWL}}

In this section, we prove that fCWL is at least as expressive as CWL and $1$-WL on cell complexes that kept node set.

Once this is established, the remaining correspondences between models with injective aggregation and their associated tests follow identically from the proof of equivalence between CWL and CWN in \citep{bodnar2021weisfeilercell}.

\begin{prop}
    fCWL is at least as expressive as sCWL.
\end{prop}
\begin{proof}
    $(\mathcal V_1, \mathcal X_1)$ and $(\mathcal V_2, \mathcal X_2)$ correspond to two cell complexes that keep node sets.

    Let $a^t$ denote the coloring at step $t$ using sCWL, and $b^t$ the coloring at step $t$ using fCWL.

    We prove by induction that $b^t \subseteq a^t$.

    \textit{Base case.} $b^0 \subseteq a^0$ since both follow the same initialization scheme.

    \textit{Induction step.} Assume $b^t \subseteq a^t$. We show that $b^{t+1} \subseteq a^{t+1}$.

    Let $(\sigma_1, \sigma_2) \in \mathcal X_1 \times \mathcal X_2$ such that $b^{t+1}_{\sigma_1} = b^{t+1}_{\sigma_2}$.  
    By construction, we have:
    $$
        b^{t}_{\sigma_1} = b^{t}_{\sigma_2}, 
        \quad 
        b^{t}_{\mathcal B}(\sigma_1) = b^{t}_{\mathcal B}(\sigma_2), \quad b^t_\mathcal C(\sigma_1) = b^t_\mathcal C(\sigma_2).
    $$
    Using Proposition~\ref{prop:include_col_hash} with the induction hypothesis, it follows that:
    $$
        a^{t}_{\sigma_1} = a^{t}_{\sigma_2},
        \quad 
        a^{t}_{\mathcal B}(\sigma_1) = a^{t}_{\mathcal B}(\sigma_2), \quad a^t_\mathcal C(\sigma_1) = a^t_\mathcal C(\sigma_2)
    $$
    \textit{i.e.}, $a^{t+1}_{\sigma_1} = a^{t+1}_{\sigma_2}$.

    By induction, $b^t \subseteq a^t$ for all $t \in \mathbb N$, hence $b \subseteq a$.
\end{proof}

Since sCWL is as expressive as CWL (Theorem~\ref{thm:expressivity_scwl}), it follows as a corollary that fCWL is at least as expressive as CWL.

\begin{prop}
    fCWL is at least as expressive as WL
\end{prop}
\begin{proof}
    
$(\mathcal V_1, \mathcal X_1)$ and $(\mathcal V_2, \mathcal X_2)$ correspond to two cell complexes that keep node sets.

    Let $a^t$ denote the coloring of nodes at step $t$ using WL, $b^t$ the coloring of cells at step $t$ using fCWL, and $b^t_{\mathcal V}$ the coloring of nodes in the cell complex colored at step $t$ by fCWL.

    We prove by induction that $b^t_{\mathcal V} \subseteq a^t$ on the nodes.

    \textit{Base case.} $b^0 \subseteq a^0$ since both initialize all nodes with the same constant color.

    \textit{Induction step.} Assume $b^t_{\mathcal V} \subseteq a^t$ on nodes. We show that $b^{t+1}_{\mathcal V} \subseteq a^{t+1}$.

    Let $(i_1, i_2) \in \mathcal V_1 \times \mathcal V_2$ such that $b^{t+1}_{i_1} = b^{t+1}_{i_2}$.

    We have:
        $$b^{t}_{i_1} = b^{t}_{i_2}, 
        \quad 
        b^{t}_{\mathcal C(i_1)} = b^{t}_{\mathcal C(i_2)}, 
        \quad 
        b^{t}_{\mathcal N(i_1)} = b^{t}_{\mathcal N(i_2)}.$$
    Using the induction hypothesis: $a^t_{i_1} = a^t_{i_2}$.
    as $b^{t}_{\mathcal N(i_1)} = b^{t}_{\mathcal N(i_2)}$, we can only consider the color of the first component, we get:
    $$\{\{b_j^t, \quad j \in \mathcal N(i_1)\}\} = \{\{b_j^t, \quad j \in \mathcal N(i_2)\}\},$$
    \textit{i.e.}, by using proposition \ref{prop:include_col_hash} and the induction hypothesis: 
    $$\{\{a_j^t, \quad j \in \mathcal N(i_1)\}\} = \{\{a_j^t, \quad j \in \mathcal N(i_2)\}\}.$$
    From WL update, we get $a_{i_1}^{t+1} = a_{i_2}^{t+1}$.    

    By induction. $b^t_{\mathcal V} \subseteq a^t$ for all $t \in \mathbb N$, thus $b_{\mathcal V} \subseteq a$.
\end{proof}

\subsection{Proof of Proposition \ref{prop:model_complexity}}

In this section, we analyse the theoretical time and memory complexity of CWN, fCWN, and sCWN on maximal clique complexes.
We first remind some notations:
\begin{itemize}
    \setlength\itemsep{0pt}
    \item $\mathcal V$ represents the set of nodes
    \item $n$ is the number of nodes of our graphs
    \item $\mathcal N_i$ represents the neighborhood of node $i$. 
    \item $\mathcal X$ is the set of maximal cliques.
\end{itemize}
We now detail one by one each message passing scheme's complexity. 

\textbf{Boundary messages.} Each node in the graph sends a message to the clique containing it. The total number of messages sent is: 
$$|\{(i,\sigma) \in \mathcal V \times \mathcal X, \quad i \in \sigma \}| = \underset{(i,\sigma) \in \mathcal V \times \mathcal X}{\sum} \mathbb 1_{i \in \sigma} = \underset{\sigma \in \mathcal X}{\sum} \underset{i \in \mathcal V}{\sum} \mathbb 1_{i \in \sigma} = \underset{\sigma \in \mathcal X}{\sum} |\sigma|.$$

\textbf{Co-boundary messages.} Each clique sends a message to each node it contains.
The total number of messages sent is: 
$$|\{(i,\sigma) \in \mathcal V \times \mathcal X, \quad i \in \sigma \}|  = \underset{\sigma \in \mathcal X}{\sum} |\sigma|.$$

\textbf{Upper-adjacency CWN.} Each node $i$ aggregate message for all tuple $(j,\sigma)$ such that $\{i,j\} \subset \sigma$. The total number of messages sent is:
\begin{align*}
\sum_{i \in \mathcal V} 
   \bigl|\{(j,\sigma) \in \mathcal V \times \mathcal X : \{i,j\} \in \sigma\}\bigr|
&= \sum_{i \in \mathcal V}\sum_{j \in \mathcal V}\sum_{\sigma \in \mathcal X} 
   \mathbb{1}_{\{i,j\} \subset \sigma} \\[6pt]
&= \sum_{\sigma \in \mathcal X}\sum_{i \in \mathcal V}\sum_{j \in \mathcal V} 
   \mathbb{1}_{\{i,j\} \subset \sigma} \\[6pt]
&= \sum_{\sigma \in \mathcal X} 
   \bigl|\{(i,j) \in \mathcal V^2 : \{i,j\} \subset \sigma\}\bigr| \\[6pt]
&= \sum_{\sigma \in \mathcal X} \binom{|\sigma|}{2} \\[6pt]
&= \sum_{\sigma \in \mathcal X} \dfrac{|\sigma|^2 - |\sigma|}{2}.
\end{align*}

\textbf{Upper-adjacency fCWN.} For each tuple $(i,\sigma) \in \mathcal V \times \mathcal X$ we create a message. Then we do an adjacency update. The total number of messages is the sum of each: 
$$\underset{(i,\sigma) \in \mathcal V \times \mathcal X}{\sum}\mathbb 1_{i \in \sigma} + \underset{l \in \mathcal N_i}{\sum}1 = \underset{\sigma \in \mathcal X}{\sum} |\sigma| + |\mathcal E|.$$

We can now finish the proof of proposition \ref{prop:model_complexity}.

\textbf{CWN.} Every message passes through an MLP $M_\uparrow$. The memory complexity is the same as the number of messages plus the data on the node and cliques:
\begin{itemize}
    \setlength\itemsep{-1pt}
    \item Time complexity : $\mathcal O(\underset{\sigma \in \mathcal X}{\sum} |\sigma|^2)$.
    \item Memory complexity : $\mathcal O(n +\underset{\sigma \in \mathcal X}{\sum} |\sigma|^2)$.
\end{itemize}

\textbf{fCWN.} Only the first messages go through an MLP $M_\uparrow$. 
\begin{itemize}
    \setlength\itemsep{-1pt}
    \item Time complexity : $\mathcal O(\underset{\sigma \in \mathcal X}{\sum} |\sigma| + |\mathcal E|)$.
    \item Memory complexity : $\mathcal O(n + \underset{\sigma \in \mathcal X}{\sum} |\sigma|)$.
\end{itemize}
\textbf{sCWN.} Here, MLPs are only applied to node or clique data. The messages are based on boundary and co-Boundary.
\begin{itemize}
    \setlength\itemsep{-1pt}
    \item Time complexity : $\mathcal O(\underset{\sigma \in \mathcal X}{\sum}|\sigma|)$.
    \item Memory complexity : $\mathcal O(n + |\mathcal X|)$.
\end{itemize}

\paragraph{Summary.}  
For clarity, we summarize below:
\begin{center}
\begin{tabular}{lcc}
\toprule
Model & Time Complexity & Memory Complexity \\
\midrule
CWN   & $\displaystyle \mathcal O(\sum_{\sigma \in \mathcal X} |\sigma|^2)$ & $\displaystyle \mathcal O(n + \sum_{\sigma \in \mathcal X} |\sigma|^2)$ \\
fCWN  & $\mathcal O(\displaystyle \sum_{\sigma \in \mathcal X} |\sigma| + |\mathcal E|)$ & $ \displaystyle \mathcal O(n + \sum_{\sigma\in \mathcal X} |\sigma|)$ \\
sCWN  & $\displaystyle\mathcal O(\underset{\sigma \in \mathcal X}{\sum}|\sigma|)$ & $\displaystyle\mathcal O(n + |\mathcal X|)$ \\
\bottomrule
\end{tabular}
\label{tab:mem_time_comp}
\end{center}

\subsection{Proof of Proposition \ref{prop:maximal_clique}}
We show that at every step of Algorithm~\ref{alg:cliquewalk}, the nodes in the walk always form a clique.  

\textbf{Notations.} Let $\text{Walk}_t$ denote the nodes in the walk at step $t$, and $\text{neighbor}_t$ the set of nodes that can be added next. We claim that:
\[
\text{neighbor}_t = \{ l \in \mathcal V ,\quad l \sim j \ \forall j \in \text{Walk}_t \},
\]
\textit{i.e.}, $\text{neighbor}_t$ contains exactly the nodes connected to all nodes in the current walk.

\textbf{Induction.}

\textit{Base case.}  
Initially, $\text{Walk}_0 = [i]$ and $\text{neighbor}_0 = \mathcal N_i$.  
By definition, $\mathcal N_i$ contains all nodes connected to $i$, i.e., all nodes that form a clique with $\text{Walk}_0$.  
Thus, the property holds at the first step.

\textit{Inductive step.}  
Assume the property holds at step $t$, and let $j_{\text{new}} \in \text{neighbor}_t$ be the next node added to the walk.  
The neighbor set is updated as
\[
\text{neighbor}_{t+1} = \text{neighbor}_t \cap \mathcal N_{j_{\text{new}}}.
\]  
By construction, $\text{neighbor}_{t+1}$ contains only nodes connected to $j_{\text{new}}$ and to all nodes in $\text{Walk}_t$, i.e., nodes connected to all nodes in 
\[
\text{Walk}_{t+1} = \text{Walk}_t \cup \{ j_{\text{new}} \}.
\]  
The property holds at step $t+1$.

\textbf{Conclusion.}  
By induction, all nodes in the walk are connected to each other, \textit{i.e.}, the walk always forms a clique.
Since the walk is a clique, its size cannot exceed $\omega(G)$, the size of the largest clique in the graph.  
Therefore, the walk can only stop when $\text{neighbor}_t$ becomes empty, \textit{i.e.}, when there is no node that can be added to extend the clique.  
As a result, the clique produced by the walk is maximal with respect to set inclusion.

\subsection{Proof of Proposition \ref{prop:complexity_cliquewalk}}

CliqueWalk builds a maximal clique by growing it step by step.  
At each step, the algorithm: (i) samples a neighbor, (ii) intersects the neighborhoods of the current and newly visited node to restrict the walk, and (iii) continues until either the walk length reaches $\omega_{\text{max}}$ or it cannot be expanded.

We can now break down the cost of one walk:

\begin{enumerate}[label=(\roman*), leftmargin=0.7cm]
    \setlength\itemsep{-1pt}
    \item \textit{Neighbor sampling.}  
    Selecting a random neighbor is constant-time: $O(1)$.

    \item \textit{Neighborhood intersection.}  
    Intersecting two neighborhoods $A$ and $B$ takes $O(|A| + |B|)$.  
    Since each neighborhood is bounded by the maximum degree $d_{\max}(G)$, this step costs at most $O(d_{\max}(G))$.

    \item \textit{Walk length.}  
    The maximum length of a walk is bounded by
    \[
    L \leq \max\bigl(\omega(G), \, \omega_{\max}\bigr),
    \]
    where $\omega(G)$ is the maximum clique size of the graph and $\omega_{\max}$ is the cutoff imposed by the algorithm.
\end{enumerate}

The complexity of one CLiqueWalk is thus :
$$\displaystyle\mathcal O (\underset{j =0}{\overset{L}{\sum}} d_{\text{max}(G)}) = \mathcal O \left (d_{\text{max}}(G) \cdot \max(\omega(G), \omega_{\text{max}}) \right ).$$

As we launch from each node $n_{\text{walks}}$ walks,  the total complexity is
\[
O\!\left(n \cdot n_{\text{walks}} \cdot d_{\max}(G) \cdot \max(\omega(G), \, \omega_{\max})\right).\qed
\]
\begin{figure}[h!]
  \centering
  \includegraphics[width=0.7\textwidth]{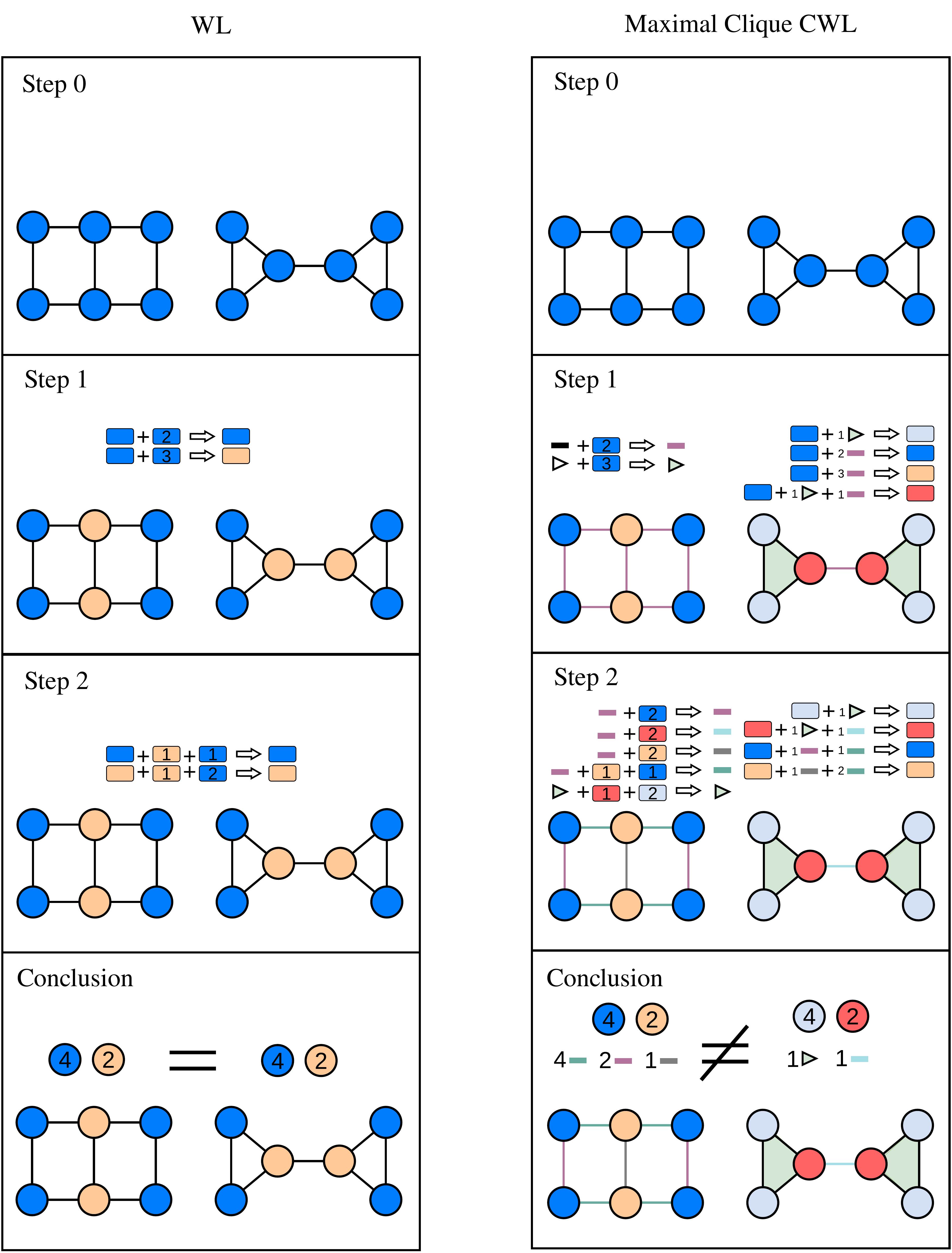}
  \caption{{Illustration of the {WL} and maximal clique {CWL} test. At each iteration, every node updates its color based on its own color and the colors of its neighboring structures (see Steps~1 and~2). After Step~2, the colors become stable (\ie invariant under further updates), and the algorithm stops.  A histogram of colors is then computed. Since the two graphs produce identical histograms for {WL}, the test cannot distinguish between them, and the WL test is therefore inconclusive. In contrast, the maximal-clique CWL algorithm yields different histograms for the two graphs, allowing us to conclude that they are not isomorphic.}}
  \label{fig:CWL_ex}
\end{figure}

\newpage

\section{Conjecture discussion}
\label{app:conj}

Let $G = (\mathcal{V}, \mathcal{E})$ be a graph and $\mathcal{X}$ the set of its maximal cliques. From Theorem~2, Cellular Weisfeiler--Lehman (CWL) is equivalent to its simplified variant (sCWL) on the same complex. This yields an equivalent formulation as 1-WL applied to the bipartite incidence graph
\[
B = (\mathcal{V} \cup \mathcal{X}, \mathcal{E}_B),
\]
where $(v, C) \in \mathcal{E}_B$ if and only if $v \in C$.

Under this formulation, node colors evolve as
\[
c_v^{t+1} = \mathrm{HASH}\Big(c_v^t,\; \{\!\{ c_{\sigma}^t \mid \sigma \in \mathcal{X}, v \in \sigma \}\!\}\Big),
\]
and clique colors are defined by
\[
c_C^t = \mathrm{HASH}\big(\{\!\{ c_u^{t-1} \mid u \in C \}\!\}\big).
\]

\paragraph{A Sufficient Condition for CWL $\supseteq$ 1-WL}

A key question is whether CWL is at least as expressive as 1-WL. This depends on whether neighborhood information can be recovered from clique-based aggregation.

\begin{defi}[C2N Identification Property]
A vertex $v \in \mathcal{V}$ satisfies the \textbf{C2N identification property} if the multiset of neighbor colors
\[
\{\!\{ c_u \mid u \in \mathcal{N}(v) \}\!\}
\]
is uniquely determined by the multiset of colors of maximal cliques containing $v$,
\[
\{\!\{ c_C \mid v \in C \}\!\}.
\]
\end{defi}

\begin{prop}
If every vertex in two graphs $G$ and $G'$ satisfies the C2N identification property, then CWL is at least as expressive as 1-WL for distinguishing $G$ and $G'$.
\end{prop}

\begin{proof}[Proof sketch]
The proof proceeds by induction over refinement steps. Assuming CWL colors are at least as refined as 1-WL at step $t$, the C2N property ensures that clique-based aggregation at step $t+2$ uniquely determines the 1-WL neighborhood histogram at step $t$, thereby recovering the 1-WL update at step $t+1$.
\end{proof}

\paragraph{A Structural Regime Where C2N Holds.} For graphs with simple clique structures, we can naturally show that the C2N property.

\begin{prop}
If every edge $(u,v) \in \mathcal{E}$ belongs to exactly one maximal clique, then the C2N identification property holds for all vertices.
\end{prop}

\begin{proof}
In this setting, the neighborhood of a vertex $v$ decomposes as a disjoint union of the maximal cliques containing $v$ (excluding $v$ itself). Let $\mathcal{X}_v = \{\sigma \in \mathcal{X} \mid v \in \sigma\}$.

Suppose two node colorings $c$ and $c'$ are consistent with the same maximal clique coloring:
\[
\{\!\{ c_\sigma \mid \sigma \in \mathcal{X}_v \}\!\}
=
\{\!\{ c'_\sigma \mid \sigma \in \mathcal{X}_v \}\!\}.
\]
Substituting the definition of clique colors,
\[
\{\!\{ \{\!\{ c_w \mid w \in \sigma \}\!\} \mid \sigma \in \mathcal{X}_v \}\!\}
=
\{\!\{ \{\!\{ c'_w \mid w \in \sigma \}\!\} \mid \sigma \in \mathcal{X}_v \}\!\}.
\]
Taking the union over all such cliques (excluding $v$) yields
\[
\bigcup_{\sigma \in \mathcal{X}_v} \{\!\{ c_w \mid w \in \sigma \setminus \{v\} \}\!\}
=
\bigcup_{\sigma \in \mathcal{X}_v} \{\!\{ c'_w \mid w \in \sigma \setminus \{v\} \}\!\},
\]
which implies
\[
\{\!\{ c_w \mid w \in \mathcal{N}(v) \}\!\}
=
\{\!\{ c'_w \mid w \in \mathcal{N}(v) \}\!\}.
\]
Thus, the neighborhood color histogram is uniquely determined.
\end{proof}

\paragraph{Clique Partition Ambiguity.} The C2N property does not hold in general. At $n=9$, we identify a structural failure where overlapping cliques obscure neighborhood information.

In particular, there exist vertices that:
\begin{itemize}
    \item are distinguished by 1-WL through different neighbor color histograms,
    \item but receive identical CWL colors due to identical multisets of incident clique colors.
\end{itemize}

We refer to this phenomenon as \textbf{clique partition ambiguity}. It shows that recovering neighborhood information from clique structure is not always locally possible, preventing a direct inductive proof that CWL implies 1-WL.

\paragraph{Empirical Evaluation for $n \leq 9$.} We performed an exhaustive evaluation over all connected, non-isomorphic graphs up to $n=9$ (261,080 graphs in total).

\begin{table}[h]
\centering
\caption{Isomorphism resolution results for 1-WL and CWL.}
\begin{tabular}{crcc}
\toprule
$n$ & Total Graphs & 1-WL Identifies & CWL Identifies \\
\midrule
1--5 & 31 & 31 & 31 \\
6 & 112 & 109 & 112 \\
7 & 853 & 836 & 853 \\
8 & 11,117 & 10,897 & 11,106 \\
9 & 261,080 & 258,632 & 260,960 \\
\bottomrule
\end{tabular}
\end{table}

CWL consistently matches or exceeds the distinguishing power of 1-WL and strictly improves upon it for $n \geq 6$. Importantly, we observe no counterexample where 1-WL distinguishes two graphs that CWL fails to separate.

These results provide strong empirical evidence for the conjecture
\[
\mathrm{CWL} \;\supseteq\; \mathrm{1\text{-}WL}.
\]
From a theoretical perspective, the conjecture remains non-trivial. CWL operates over maximal cliques rather than direct neighborhoods, and clique partition ambiguity shows that neighborhood information is not always locally recoverable. 

Nevertheless, such local ambiguities do not appear to translate into global failures for graph isomorphism in our experiments. Establishing or refuting the conjecture in full generality remains an interesting non trivial open problem.

\section{Maximal Clique CWL\label{app:cwl_maximal}}

{We propose some experiments and illustrations to better understand the maximal clique CWL and its differences with WL. See Figure~\ref{fig:CWL_ex}.}
It is known that CWL is more expressive than WL when using cell lifting methods that preserve the full node and edge sets of the graph~\citep{bodnar2021weisfeilercell}. However, since we only consider maximal cliques and remove edges from the representation, we no longer have this guarantee over WL.


We introduce two simple coloring schemes to make sense of CWL's expressive power.
\begin{table}
\centering
\caption{Number of distinct hashes found by each method on graph classification datasets. Abbreviated dataset names: ENZ = ENZYMES, FRANK = FRANKENSTEIN, IMDB-B = IMDB-BINARY, IMDB-M = IMDB-MULTI, PROT = PROTEINS, ALC = alchemy\_full.}
\label{tab:distinct_hashes}
\resizebox{0.8\linewidth}{!}{%
\begin{tabular}{lccccccccc}
    \toprule
    \textbf{Method} & \textbf{DD} & \textbf{ENZ} & \textbf{FRANK} & \textbf{IMDB-B} & \textbf{IMDB-M} & \textbf{NCI1} & \textbf{PROT} & \textbf{ALC} \\
    \midrule
    1WL & 1178 & 595 & 2766 & 537 & 387 & 3837 & 996 & 12343 \\
    CWL & 1178 & 595 & 2767 & 537 & 387 & 3837 & 996 & 12396 \\
    CountClique & 1178 & 547 & 216 & 432 & 309 & 254 & 799 & 23 \\
    TopoCount & 1178 & 595 & 1272 & 537 & 387 & 2188 & 992 & 727 \\
    \bottomrule
\end{tabular}%
}
\end{table}

\begin{table}
\centering
\caption{Number of graphs in each strongly regular family.}
\label{tab:sr_num_graphs}
\begin{minipage}[t]{0.45\linewidth}
\centering
\begin{tabular}{lc}
\toprule
\textbf{Family} & \textbf{Number of graphs} \\
\midrule
16-6-2-2 & 2 \\
25-12-5-6 & 15 \\
26-10-3-4 & 10 \\
28-12-6-4 & 4 \\
29-14-6-7 & 41 \\
35-18-9-9 & 3854 \\
\bottomrule
\end{tabular}
\end{minipage}%
\begin{minipage}[t]{0.45\linewidth}
\centering
\begin{tabular}{lc}
\toprule
\textbf{Family} & \textbf{Number of graphs} \\
\midrule
36-14-4-6 & 180 \\
40-12-2-4 & 28 \\
45-12-3-3 & 78 \\
50-21-8-9 & 18 \\
64-18-2-6 & 167 \\
\bottomrule
\end{tabular}
\end{minipage}
\end{table}

\begin{defi}
    The \textbf{CountClique} test hashes the set of all clique lengths.
\end{defi}
\begin{defi}
    The \textbf{TopoCount} test assigns a unique color to each node by hashing the set of lengths of the cliques containing it.
\end{defi}
It is clear that CWL is at least as expressive as TopoCount and CountClique.

We empirically compare the expressivity of CWL, WL, and other tests on various datasets.
Table~\ref{tab:distinct_hashes} shows the number of distinct hashes produced by each method. 
CWL matches or slightly exceeds WL in most cases.
For several datasets~\citep{dobson2003distinguishing, chen2019alchemyquantumchemistrydataset, Orsini2015frankenstein}, access to clique neighborhood information allows CWL to distinguish more graphs.
For chemical datasets such as \textit{alchemy\_full}, WL schemes produce significantly more hashes than one-shot methods like TopoCount, highlighting the benefit of multi-layer models on those datasets.

\begin{figure}[h!]
    \centering
    \begin{subfigure}[b]{0.48\textwidth}
        \centering
        \includegraphics[width=0.95\textwidth]{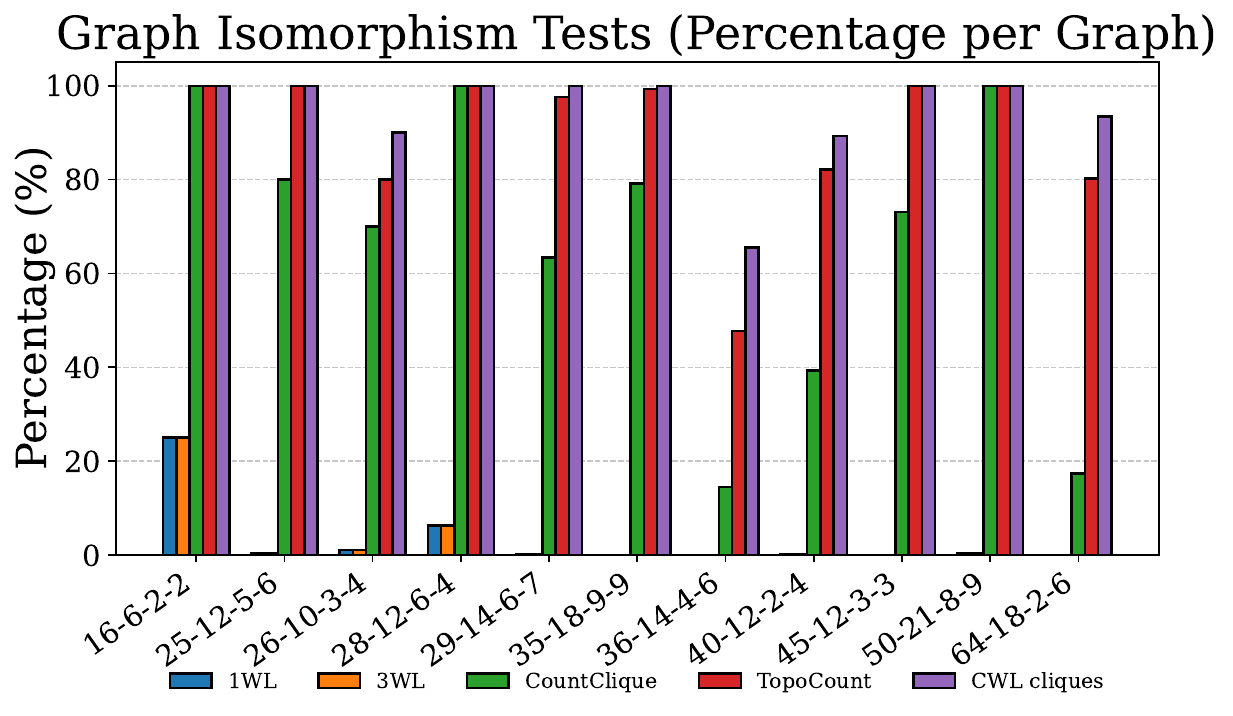}
        \caption{Clique CWL against other tests}
        \label{fig:hash_tests}
    \end{subfigure}
    \hspace{1pt}
    \begin{subfigure}[b]{0.48\textwidth}
        \centering
        \includegraphics[width=0.95\textwidth]{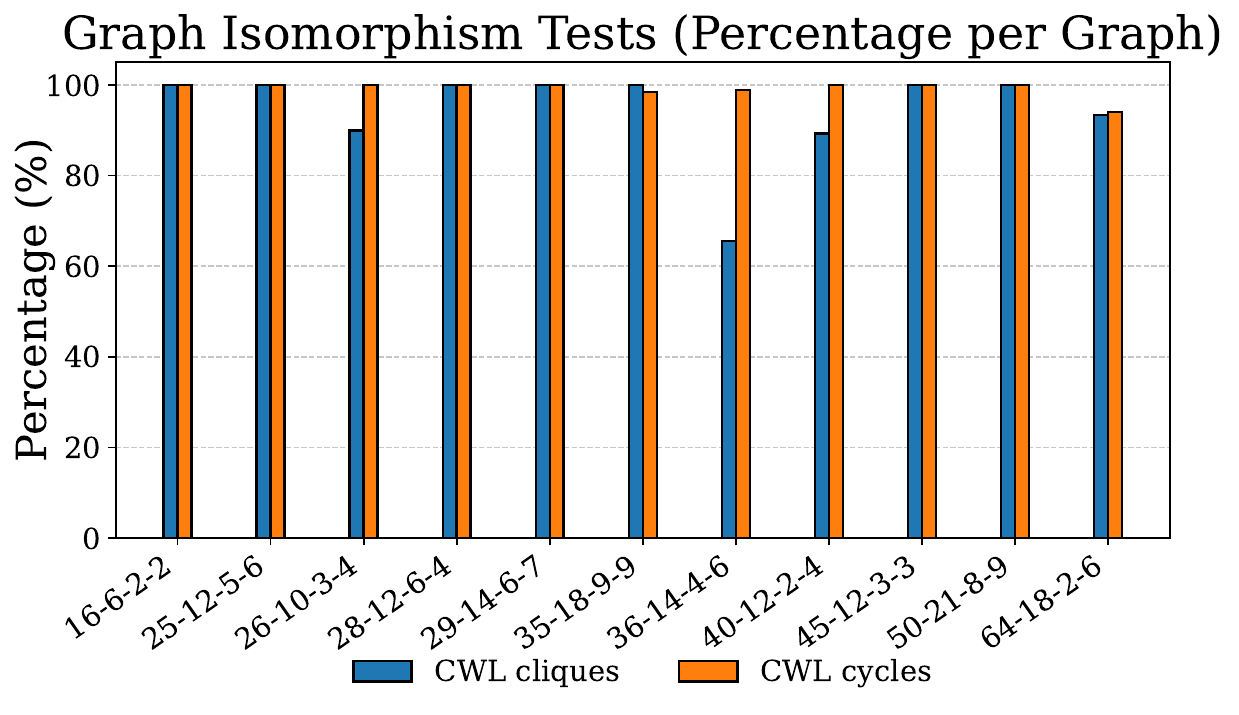}
        \caption{Clique CWL against cycle CWL}
        \label{fig:hash_tests_cycle}
    \end{subfigure}
    \caption{Comparison of Percentage of Unique Graph Hashes on strongly regular datasets: (a) compare CWL on maximal cliques against other isomorphic tests, (b) compare CWL on maximal cliques against CWL on node, edge, cycle lifting.}
    \label{fig:clique_vs_cycle}
\end{figure}
We also evaluate these tests on strongly regular graphs (see Figure \ref{fig:hash_tests} and Table \ref{tab:sr_num_graphs}).
We use strongly regular datasets from \url{https://www.maths.gla.ac.uk/~es/srgraphs.php}~\citep{HAEMERS2001839}, which include non-isomorphic strongly regular graphs with up to 64 nodes.
For many strongly regular graph families, clique topology alone is sufficient to distinguish most graphs.
In contrast, $1$WL and $3$WL fail to discriminate any graphs in these families, which aligns with known results~\citep{Bouritsas_2023, bodnar2021weisfeilercell}.

\textbf{Clique against cycle lifting.} Figure~\ref{fig:hash_tests_cycle} compares CWL with node and maximal clique lifting against CWL with node, edge, and cycle lifting. Both approaches achieve similar graph discriminative power, though they are not directly comparable: in some cases, cliques distinguish more graphs, while in others, cycles do.


{
A classical approach for enumerating all maximal cliques is the \textit{Bron-Kerbosch} method~\citep{Bron_1973}, explained in Algorithm~\ref{alg:bronkerbosch}. 
$R$ is the current clique being grown, 
$P$ contains nodes adjacent to all vertices in $R$, and $X$ contains nodes already processed that are also adjacent to every vertex in $R$.
Clique summarization has been widely studied~\citep{delia25}. Most of those approaches modify the Brun-Kerbosch algorithm to enumerate or sample a subset of the maximal clique set that verifies specific properties. For instance, \citet{Wang13MaxCliqueSampling} prunes branches based on a heuristic to construct a subset of maximal cliques that partially covers all maximal cliques.}

\paragraph{BREC Benchmark.} \label{sec:brec} 
To provide a standardized comparison of expressivity, we evaluate CWN on the
maximal clique complex against the BREC benchmark~\citep{wang2023brec}, which
tests the ability of GNNs to distinguish 400 structurally diverse graph pairs
across four categories: Basic, Regular, Extension, and CFI.

\begin{table}[h]
  \centering
  \footnotesize
  \caption{%
    BREC benchmark results: number of correctly distinguished graph pairs
    (out of the total per category).
    CWN is evaluated on the maximal clique complex.%
  }
  \label{tab:brec}
  \begin{tabular}{llccccc}
    \toprule
    \textbf{Category} & \textbf{Model}
      & \textbf{Basic (60)}
      & \textbf{Regular (140)}
      & \textbf{Ext.\ (100)}
      & \textbf{CFI (100)}
      & \textbf{Total (400)} \\
    \midrule
    \multirow{1}{*}{Ours}
      & CWN (max.\ clique) & 55 & 119 & 28 & 0 & 202 \\
    \midrule
    \multirow{2}{*}{$k$-WL GNNs}
      & PPGN        & 60 & 50  & 100 & 23 & 233 \\
      & $\Delta$-LGNN & 60 & 50  & 100 & 6  & 216 \\
    \midrule
    \multirow{2}{*}{Subgraph GNNs}
      & NGNN        & 59 & 48  & 59  & 0  & 166 \\
      & DS-GNN      & 58 & 48  & 100 & 16 & 222 \\
    \midrule
    Substructure  & GSN         & 60 & 99  & 95  & 0  & 254 \\
    Transformer   & Graphormer  & 16 & 12  & 41  & 10 & 79  \\
    Non-GNN       & N2          & 60 & 138 & 100 & 0  & 298 \\
    \bottomrule
  \end{tabular}
\end{table}

CWN on the maximal clique complex achieves particularly strong discrimination on
Regular graphs, correctly distinguishing 119 out of 140 pairs overall, and
improving to 119 out of 120 pairs when restricted to regular graphs of degree
greater than 2 (i.e., those containing cliques larger than single edges).
This places CWN as the second-best method on this category among those evaluated.
Performance is weaker on CFI graphs and Extension graphs.
The CFI result is expected: CFI graphs contain no cliques larger than edges,
effectively reducing CWN to a 1-WL baseline.
Extension graphs have an average maximal clique size of 2.9, suggesting that
the exclusive reliance on maximal cliques is insufficient to separate these
specific structures, pointing to a direction for future work.

\begin{algorithm}[t]
\caption{CliqueWalk}
\label{alg:cliquewalk}
\begin{algorithmic}[1]
\Function{CliqueWalk}{$i, \mathcal{N}, \omega_{max}$}
    \State $Walk \gets [i]$
    \State $neighbor \gets \mathcal{N}_i$
    \While{$neighbor \neq \emptyset$ \quad $|Walk| < \omega_{max}$}
        \State choose $j \in neighbor$
        \State append $j$ to $Walk$
        \State $neighbor \gets neighbor \cap \mathcal{N}_j$
    \EndWhile
    \State \Return $Walk$
\EndFunction
\end{algorithmic}
\end{algorithm}

\begin{algorithm}[h]
\caption{Bron--Kerbosch}
\label{alg:bronkerbosch}
\begin{algorithmic}[1]
\Function{BronKerbosch}{$R, P, X$}
    \If{$P = \emptyset$ \quad $X = \emptyset$}
        \State report $R$ as a maximal clique
    \Else
        \For{each $u$ in a copy of $P$}
            \State $P \gets P \setminus \{u\}$
            \State $R_{\text{new}} \gets R \cup \{u\}$
            \State $P_{\text{new}} \gets P \cap N(u)$
            \State $X_{\text{new}} \gets X \cap N(u)$
            \State \Call{BronKerbosch}{$R_{\text{new}}, P_{\text{new}}, X_{\text{new}}$}
            \State $X \gets X \cup \{u\}$
        \EndFor
    \EndIf
\EndFunction
\end{algorithmic}
\end{algorithm}
\section{{Clique Sampling}}
\label{sec:clique_sampling}

Our method, \emph{CliqueWalk}, explained in Algorithm~\ref{alg:cliquewalk}, is also inspired by Bron-Kerbosch but differs in two important ways: (\textit{i}) \emph{We sample rather than full enumeration.} CliqueWalk does not attempt to enumerate all maximal cliques but samples a subset of them. Therefore, (\textit{ii}) \emph{we do not need the $X$ set.} 
We simply grow a clique by iteratively sampling a vertex from the candidate set $P$. Conceptually, CliqueWalk performs an upward random walk in the clique complex (see Figure \ref{fig:cliquewalk}).
While exact clique sampling might require exploring a geometric number of recursive branches (see Proposition~\ref{prop:maxclique_complex}), CliqueWalk runs in linear time with respect to the number of nodes (see Proposition~\ref{prop:complexity_cliquewalk}) and efficiently produces summaries of the clique topology with the following sampling guarantees: (\textit{i}) The sampling process tends to sample larger cliques. For instance, given a node $v$ and a maximal clique $\sigma$ containing $v$, the probability of sampling $\sigma$ is at most ${(|\sigma| -1)}/{ deg(v)}$. (\textit{ii}) Performing CliqueWalk with multiple 
walks per node ensures that each node is included in several sampled cliques, which is relevant for node-level learning tasks.



\section{{Ablations}}

\begin{figure*}[t]
    \centering
    \begin{subfigure}[b]{0.44\textwidth}
        \centering
        \includegraphics[width=\textwidth]{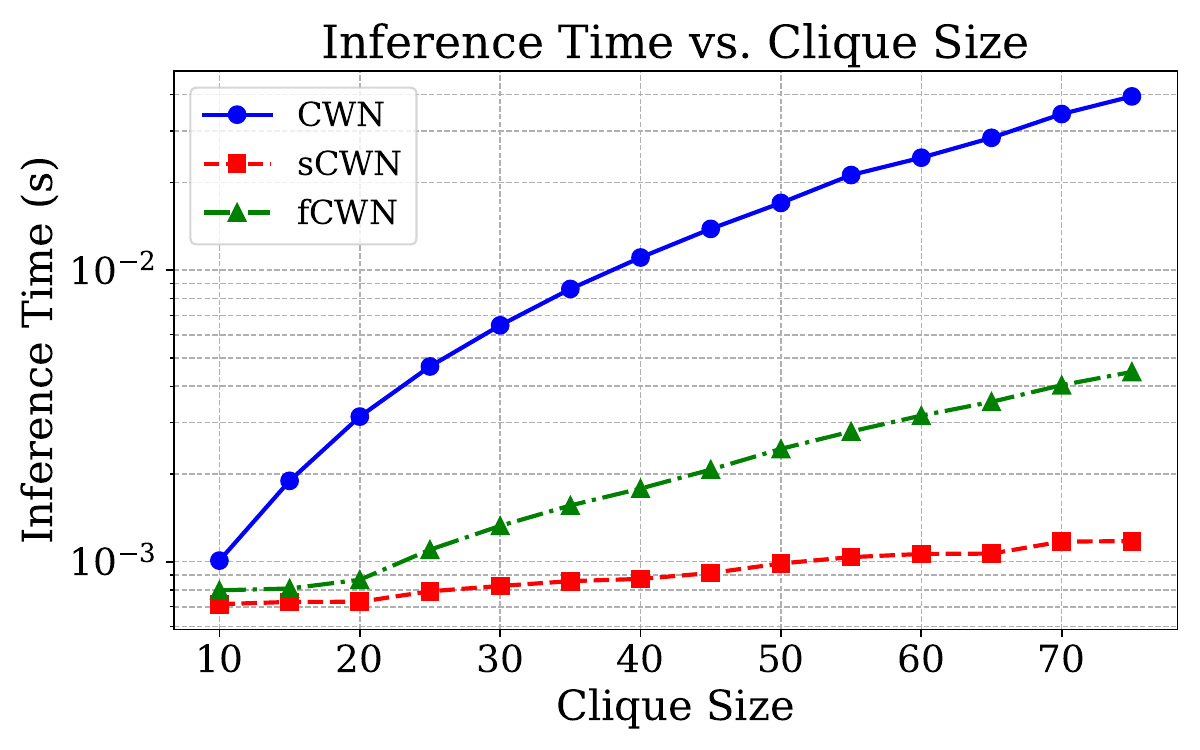}
        \caption{Inference time with growing cliques.}
        \label{fig:time_sCWL_CWL}
    \end{subfigure}
    \begin{subfigure}[b]{0.44\textwidth}
        \centering
        \includegraphics[width=\textwidth]{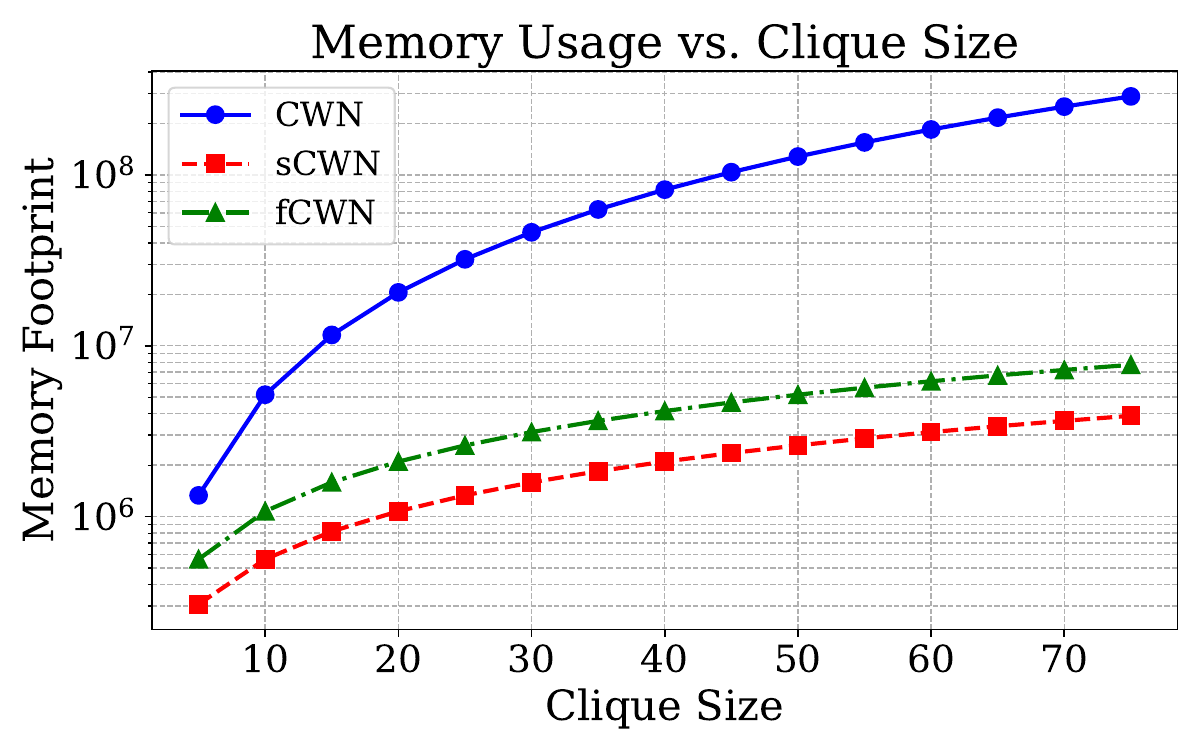}
        \caption{Memory footprint with growing cliques.}
        \label{fig:mem_sCWL_CWL}
    \end{subfigure}
    \caption{Comparison of CWN, sCWN, and fCWN with increasing clique size: (a) inference time and (b) memory footprint.}
    \label{fig:time_mem_comparison}
\end{figure*}

\label{app:ablations}
\begin{wraptable}{r}{0.45\linewidth}
\centering
\caption{Ablation cell input features. Table report test accuracy after training.}
\label{tab:ablation_input}
\resizebox{\linewidth}{!}{
\begin{tabular}{lrrr}
\toprule
input type & size emb & sum & mean \\
dataset &  &  &  \\
\midrule
Photo & 94.7\% & 94.5\% & 94.9\% \\
contact-high-school & 95.4\% & 97.6\% & 7.0\% \\
\bottomrule
\end{tabular}

}
\end{wraptable}

\textbf{Scalability of CWN models.} 
Figure~\ref{fig:time_mem_comparison} illustrates how CWN, fCWN, and sCWN scale with increasing clique size.
Consistent with Proposition~\ref{prop:model_complexity}, both fCWN and sCWN require substantially less memory and runtime than CWN.
Among them, sCWN achieves the best efficiency, confirming that restricting message passing to boundary and co-boundary relations provides a favorable tradeoff between expressivity and computational cost.

\textbf{Cell input feature choice.} {Table~\ref{tab:ablation_input} compares the performance of sCWN on \emph{Photo} and \emph{contact-primary-school} depending on the type of input used.
We observe that size embedding and sum embedding obtain very similar accuracy, whereas mean embedding provides much worse results on contact-high-school.}


\textbf{Sampling effects.} We compare exact enumeration of maximal cliques with CliqueWalk sampling using between 1 and 256 walks per node on \emph{Photo} (Figure \ref{fig:sampling_tradeoff:photo}) and \emph{contact-primary-school} (Figure \ref{fig:sampling_tradeoff:primary}). On \emph{Photo}, increasing the number of sampled cliques does not improve model accuracy. We further study more aggressive subsampling of cliques (Figure \ref{fig:sampling_tradeoff:photo_small}). We observe that fCWN is largely insensitive to the number of sampled cliques, with a slight performance increase once at least one clique per node is sampled. For sCWN, performance improves with the number of walks, but the gains become marginal beyond 4 walks per node.

On \emph{contact-primary-school}, the behavior differs between the two models. As shown in Section~\ref{subsec:ablation} (Figure~\ref{fig:walk_effect}), both fCWN and sCWN initially gain performance as the number of walks per node increases. Extending this analysis to a larger number of sampled cliques reveals the full picture: fCWN remains stable even at very high sampling rates, whereas sCWN accuracy begins to drop when more than 32 random walks per node are used. This suggests that excessive redundancy in sampled cliques can dilute useful information for sCWN, while fCWN’s message-passing scheme is more robust to the amount of sampled structure.

\textbf{Number of layers effects.} {Figure~\ref{fig:num_layers} shows the evolution of the accuracy for deeper models. As depth increases, test accuracy degrades at some point, indicating that deep models struggle to learn effectively. 
Training and testing accuracy remain similar at large depths (not shown in the figure); this decline is unlikely due to over-fitting and is consistent with the over-smoothing effect known in graph learning \cite{einizade2025continuous}.
}

\begin{figure*}[t]
    \centering
    \begin{subfigure}[b]{0.44\textwidth}
        \centering
        \includegraphics[width=\textwidth]{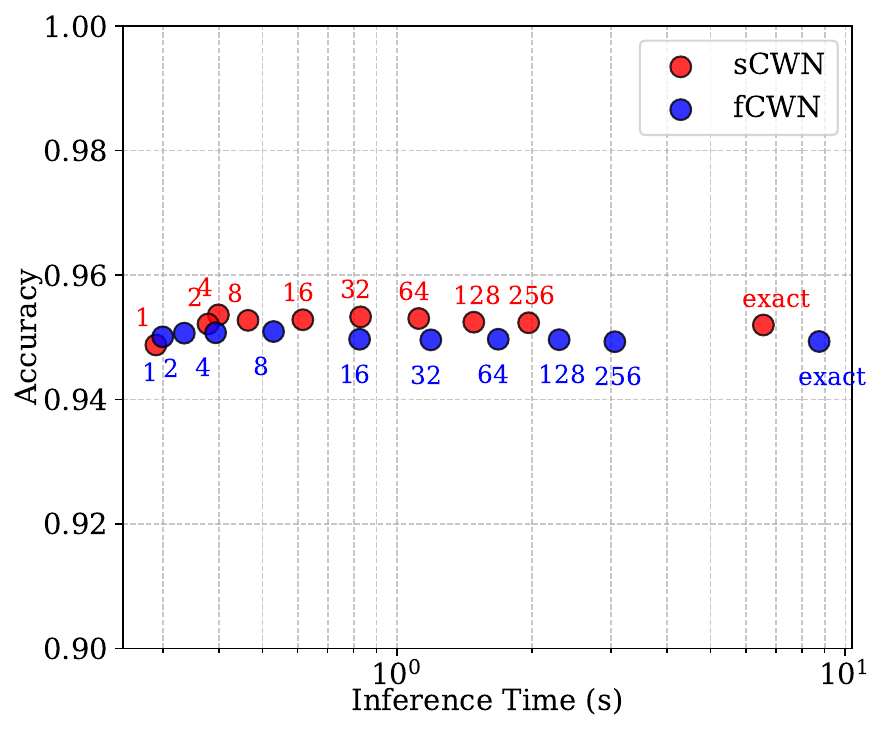}
        \caption{Accuracy of sCWN and fCWN at different CliqueWalk sampling rates on Photo.}
        \label{fig:sampling_tradeoff:photo}
    \end{subfigure}
    \hspace{0.02\textwidth}
    \begin{subfigure}[b]{0.44\textwidth}
        \centering
        \includegraphics[width=0.98\textwidth]{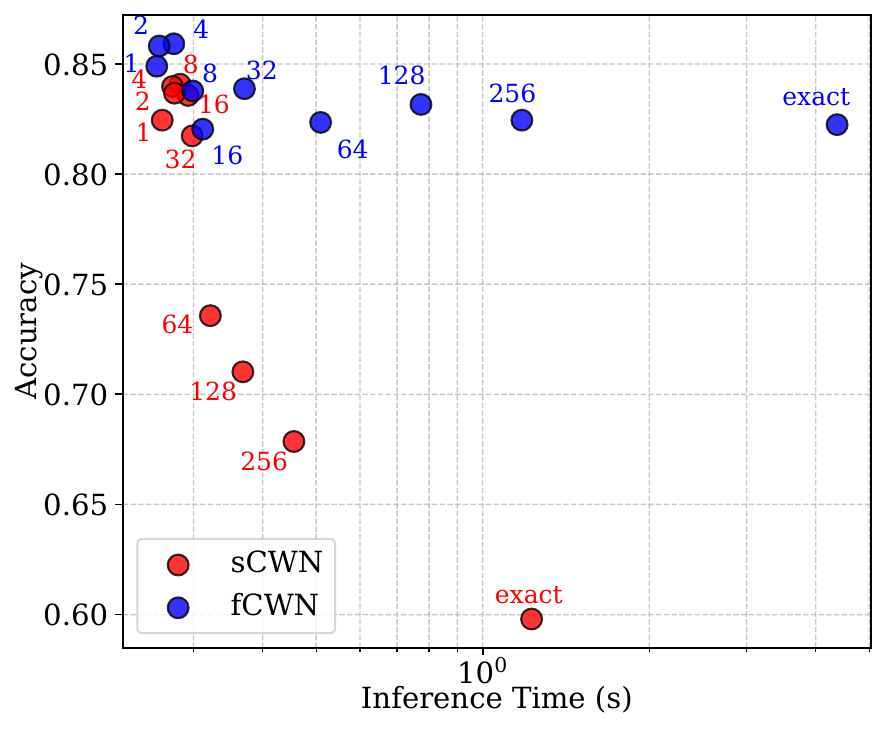}
        \caption{Accuracy of sCWN and fCWN at different CliqueWalk sampling rates on contact-primary-school.}
        \label{fig:sampling_tradeoff:primary}
    \end{subfigure}
    \caption{Ablation studies on large clique sampling versus full clique enumeration. }
    \label{fig:sensitivity_analysis}
\end{figure*}

\begin{figure*}[h]
    \centering
    \begin{subfigure}[b]{0.44\textwidth}
        \centering
        \includegraphics[width=\textwidth]{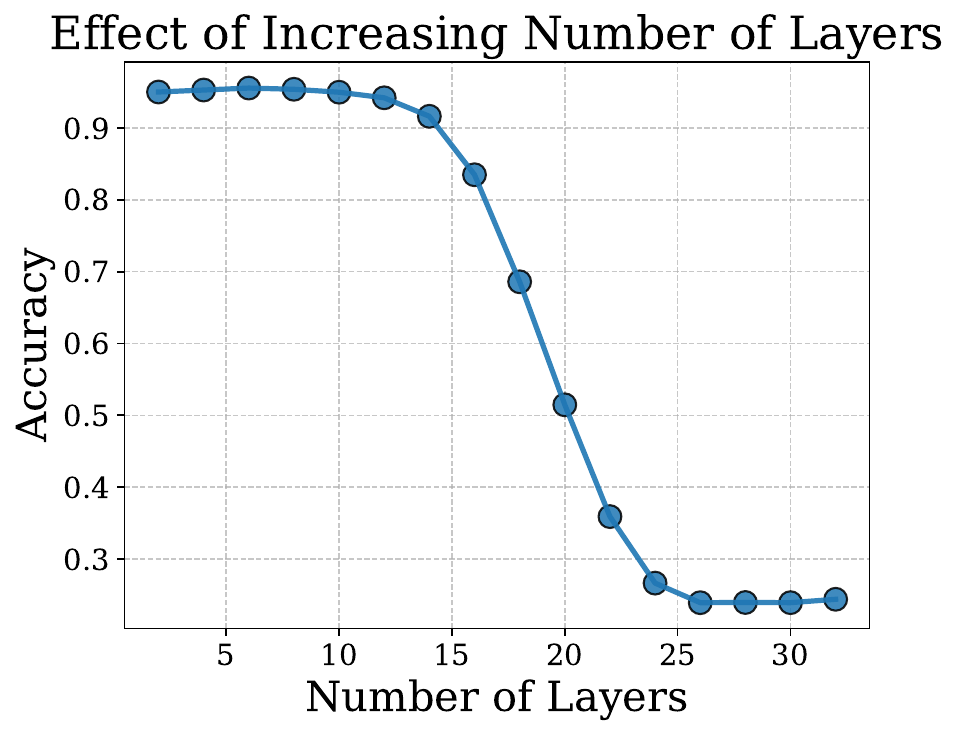}
        \caption{Accuracy of trained sCWN model without batchnorm on Photo depending on the number of layers.}
        \label{fig:num_layers}
    \end{subfigure}
    \hspace{0.02\textwidth}
    \begin{subfigure}[b]{0.44\textwidth}
        \centering
        \includegraphics[width=0.9\textwidth]{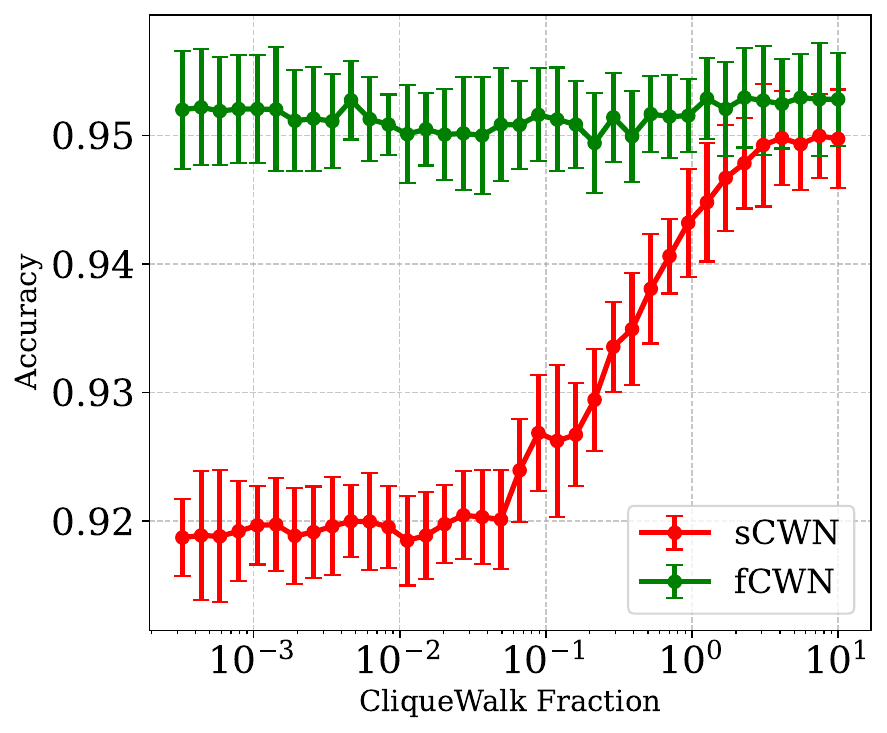}
        \caption{Accuracy of sCWN and fCWN on Photo vs. CliqueWalk sampling rates.}
        \label{fig:sampling_tradeoff:photo_small}
    \end{subfigure}
    \caption{Ablation studies on model depth and clique sampling.}
    \label{fig:sensitivity_analysis}
\end{figure*}

\begin{table}[t]
\centering
\caption{Dataset statistics for node and graph classification. 
Reported are the number of nodes, number of edges, mean degree, and clique statistics ($\mu$: mean size, $\sigma$: standard deviation).}
\label{tab:dataset_statistics}
\resizebox{0.75\textwidth}{!}{
\begin{tabular}{lrrrrr}
\toprule
\textbf{Dataset} & \textbf{Nodes} & \textbf{Edges} & \textbf{Mean degree} & \textbf{Clique $\mu$} & \textbf{Clique $\sigma$} \\
\midrule
\\[-0.8em]
\multicolumn{6}{c}{\textit{Node classification datasets}} \\
\\[-0.8em]
\midrule
SCM & 6\,002\,010 & 276\,089\,116 & 46.0 & 6.51 & 6.55 \\
Cora & 2\,708 & 10\,556 & 7.80 & 2.37 & 0.59 \\
PubMed & 19\,717 & 88\,648 & 8.99 & 2.28 & 0.59 \\
Citeseer & 3\,327 & 9\,104 & 5.47 & 2.26 & 0.58 \\
Photo & 7\,650 & 238\,162 & 62.26 & 10.75 & 4.89 \\
Contact-Primary-School & 242 & 16\,634 & 137.47 & 11.36 & 2.88 \\
Contact-High-School & 327 & 11\,636 & 71.17 & 9.28 & 3.73 \\
{OGBN-Products} & 2 449 029 & 123 718 024 & 50.5 &8.3 & 6.6 \\ 
\midrule
\\[-0.8em]
\multicolumn{6}{c}{\textit{Graph classification datasets}} \\
\\[-0.8em]
\midrule
IMDB-BINARY & 19\,773 & 96\,531 & 9.76 & 7.02 & 3.80 \\
IMDB-MULTI & 19\,502 & 98\,903 & 10.14 & 7.61 & 4.30 \\
MUTAG & 3\,371 & 3\,721 & 2.21 & 2.00 & 0.00 \\
NCI1 & 122\,747 & 132\,753 & 2.16 & 2.00 & 0.04 \\
NCI109 & 122\,494 & 132\,604 & 2.17 & 2.00 & 0.04 \\
Proteins & 43\,471 & 81\,044 & 3.73 & 2.53 & 0.63 \\
\bottomrule
\end{tabular}
}
\end{table}

\section{Permutation Equivariance of sCWN}
\label{sec:equivariance}

For a fixed cell complex, sCWN is permutation equivariant by construction.
Because CliqueWalk is a stochastic sampling procedure, however, two isomorphic
graphs may yield non-isomorphic lifted complexes for a single sample.
We now show that the \emph{expected} output ofsCWN is permutation equivariant.

\begin{prop}[Expected equivariance of sCWN]
\label{prop:equivariance}
Let $\Phi(G, \hat{\mathcal{X}})$ denote the sCWN mapping for a graph $G$
and a sampled clique complex $\hat{\mathcal{X}}$.
The expected output $\mathbb{E}[\Phi(G, \cdot)]$ is permutation equivariant.
\end{prop}

\begin{proof}
Let $\pi$ be any permutation of the vertex set $V$, and let $\pi(G)$ and
$\pi(\hat{\mathcal{X}})$ denote the permuted graph and complex, respectively.
Since CliqueWalk depends only on the adjacency structure and not on node
indexing, the sampling distribution is equivariant:
\begin{equation}
  P\!\left(\hat{\mathcal{X}} \mid G\right)
  = P\!\left(\pi(\hat{\mathcal{X}}) \mid \pi(G)\right).
\end{equation}
The expected output on the permuted graph is
\begin{equation}
  \mathbb{E}[\Phi(\pi(G), \cdot)]
  = \sum_{\hat{\mathcal{X}}'} \Phi\!\left(\pi(G), \hat{\mathcal{X}}'\right)
    P\!\left(\hat{\mathcal{X}}' \mid \pi(G)\right).
\end{equation}
Applying the change of variables $\hat{\mathcal{X}}' = \pi(\hat{\mathcal{X}})$
and the equivariance of $\Phi$ for fixed complexes
($\Phi(\pi(G), \pi(\hat{\mathcal{X}})) = \pi(\Phi(G, \hat{\mathcal{X}}))$):
\begin{align}
  \mathbb{E}[\Phi(\pi(G), \cdot)]
  &= \sum_{\hat{\mathcal{X}}}
      \pi\!\left(\Phi(G, \hat{\mathcal{X}})\right)
      P\!\left(\pi(\hat{\mathcal{X}}) \mid \pi(G)\right) \notag \\
  &= \sum_{\hat{\mathcal{X}}}
      \pi\!\left(\Phi(G, \hat{\mathcal{X}})\right)
      P\!\left(\hat{\mathcal{X}} \mid G\right) \notag \\
  &= \pi\!\left(\sum_{\hat{\mathcal{X}}}
      \Phi(G, \hat{\mathcal{X}})
      P\!\left(\hat{\mathcal{X}} \mid G\right)\right)
   = \pi\!\left(\mathbb{E}[\Phi(G, \cdot)]\right).
\end{align}
Thus, the expectation of the sCWN mapping commutes with the permutation map.
\end{proof}

\section{Datasets}
\label{app:datasets}

\textbf{Topological networks}~\citep{chodrow2021hypergraph, Mastrandrea-2015-contact}.
The \textit{contact-high-school} and \textit{contact-primary-school} datasets record proximity between students. Hyperedges are created at fixed time intervals from these interactions. We then project all interactions into a static graph. In this graph, an edge links two students if they have interacted at least once. The resulting graphs are topological complex networks (See Figures~\ref{fig:contact-primary-school} and~\ref{fig:contact-high-school})

\textbf{Citation networks.} In these datasets, node features are given by a Bag-of-Words representation of the documents.
Cora and Citeseer are citation networks extracted from machine learning publications~\citep{Sen_Namata_2008}. The labels correspond to the research topic of each paper. The PubMed citation network consists of articles related to diabetes.~\citep{Namata2012QuerydrivenAS} The labels indicate the  type of diabetes discussed in the article.

\textbf{Purchase network.} The Amazon Photo dataset is a subset of the Amazon co-purchase network~\citep{mcauley2015amazon}. In this graph, nodes represent products, and edges connect items that are frequently purchased together. node features are given by a Bag-of-Words representation of product reviews, and the labels are the product category. The OGBN-Products dataset follows the same methodology, but the Bag-of-Words features have been reduced to 100 dimensions using PCA, providing a more compact representation of the node features.

\textbf{Stochastic clique model.} 
It is a special case of Stochastic Block Model~\citep{HOLLAND1983109} where inward probability is set to $1$. Graphs are generated by assembling cliques, with nodes inside each clique fully connected. Each clique is assigned a label, which is inherited by all its nodes, and node features are generated from a Gaussian distribution with a mean determined by the node label and a fixed diagonal variance. To introduce topological noise, each node is connected to nodes outside its clique with a fixed probability, perturbing the clique structure. The task can thus be interpreted as a form of label denoising. For our experiments reported in table For experiments reported in Table~\ref{tab:node_classification}, cliques had random sizes between $10$ and $20$. Node features had a standard deviation of $2$, and topological noise was such that approximately two out of three neighbors came from outside the clique. Each clique was assigned one of five possible labels.

\textbf{Social networks.} A network of actors and actresses is constructed from IMDB, where edges indicate collaboration in the same film. The \textit{IMDB-BINARY} and \textit{IMDB-MULTI} datasets~\citep{Yanardag_2015} consist of the 1-hop neighborhoods around each actor. Graph labels correspond to the movie genre associated with the actor.

\textbf{Bioinformatics.}
The bioinformatics datasets include four widely used molecular and protein graph collections. \textit{MUTAG}~\citep{Debnath1991} contains nitroaromatic compounds with 7 different labels indicating mutagenic activity. \textit{PROTEINS}~\citep{Borgwardt_2005} represents protein structures; the task is to predict if a protein is an enzyme or not. \textit{NCI1} and \textit{NCI109}~\citep{Wale_2006, Shervashidze_2010} are collections of chemical compounds tested for activity against lung cancer and ovarian cancer cells, respectively. Each dataset is available through the TUDataset~\citep{morris2020tudatasetcollectionbenchmarkdatasets} repository and is commonly used to benchmark graph-based learning methods.


\textit{Remark.} Dataset statistics can be found in Table~\ref{tab:dataset_statistics}. 
Clique size was approximated using CliqueWalk for OGBN-Products.
Figure \ref{fig:acc_clique_cor} shows a correlation between larger mean clique sizes and better performance for clique-based models.

\textbf{OGBN-Products.} We train our model on cliques (HGNN, sCWN, fCWN) for $1000$ epochs with 200 hidden dimensions. For sCWN and fCWN, clique data are initialized using mean node features. We tune the learning rate, dropout rate, weight decay, and the number of layers. Tuning is performed on $300 $ epochs using the Optuna library with the Tree-structured Parzen Estimator algorithm and a $30$-trial budget. 

For HGNN on OGBN-Products, we used an implementation with mean aggregation instead of symmetric normalized aggregation.

For GNN models, we do not tune hyperparameters and use typical ones. 
The results we obtain in GNNs are aligned with public leaderboards for \textit{full-batch} models.

\begin{table}[h]
\centering
\caption{Hyperparameter Configurations on OGBN-products}
\label{tab:hyperparameters}
\begin{tabular}{lccccc}
\toprule
\textbf{Model} & \textbf{Layers} & \textbf{Hidden Dim} & \textbf{Learning Rate} & \textbf{Weight Decay} & \textbf{Dropout} \\ \midrule
GCN            & 4               & 200                 & $3 \times 10^{-3}$     & $1 \times 10^{-3}$    & 0.50             \\
GAT            & N/A             & N/A                 & N/A                    & N/A                   & N/A              \\
GIN            & 3               & 64                  & $1 \times 10^{-3}$     & $5 \times 10^{-4}$    & 0.50             \\
SAGEConv       & 3               & 64                  & $1 \times 10^{-3}$     & $5 \times 10^{-4}$    & 0.50             \\
SGC            & 3               & 64                  & $1 \times 10^{-3}$     & $5 \times 10^{-4}$    & 0.50             \\
SCCN           & N/A             & N/A                 & N/A                    & N/A                   & N/A              \\
HGNN           & 4               & 200                 & $2.4 \times 10^{-3}$   & $2 \times 10^{-6}$    & 0.15             \\
CWN            & N/A             & N/A                 & N/A                    & N/A                   & N/A              \\
fCWN           & 5               & 200                 & $3 \times 10^{-3}$     & $1.3 \times 10^{-6}$  & 0.17             \\
sCWN           & 5               & 200                 & $2.7 \times 10^{-3}$   & $1 \times 10^{-6}$    & 0.09             \\ \bottomrule
\end{tabular}
\end{table}

\newpage

\section{Model and Layer Details \label{app:models}}

In this section, we describe the layers and model implementations used for our benchmarks. Throughout, we use the following notation:  
\begin{itemize}
    \setlength\itemsep{-1pt}
    \item $\text{MLP}$: a 2-layer multilayer perceptron with ReLU activation.  
    \item $\mathbf{W}$: a learnable linear layer.  
    \item $\mathbf H \in \{0,1\}^{n \times m}$: the hypergraph incidence matrix.  
    \item $\mathbf D_v \in \mathbb{R}^{n \times n}$, $\mathbf D_e \in \mathbb{R}^{m \times m}$: diagonal degree matrices of nodes and hyperedges (cliques):
    \[
        \mathbf D_v(i,i) = \sum_{e=1}^{m} \mathbf H_{i,e}, \quad
        \mathbf D_e(e,e) = \sum_{i=1}^{n} \mathbf H_{i,e}.
    \]
    \item $\mathcal{X}$: the set of cliques.  
    \item $\mathbf{x}_i^N$: features of node $i \in \mathcal V$.  
    \item $\mathbf{x}_\sigma^C$: features of clique $\sigma \in \mathcal{X}$.  
\end{itemize}

\textbf{HGNN.} We follow \citep{feng2019hypergraphneuralnetworks}. The layer propagation is:  
\[
\mathbf {x}_i^N \gets \mathbf {Wx}_i^N +   \mathbf{W} \mathbf D_v^{-\frac{1}{2}} \mathbf H \mathbf D_e^{-1} \mathbf H^\top \mathbf D_v^{-\frac{1}{2}}  \mathbf W(\mathbf{x}_i^N),
\]
where $\mathbf W$ is a learnable weight matrix, and $\sigma(\cdot)$ is a non-linear activation function (\textit{e.g.}, ReLU). The addition of $\mathbf{X}^{(l)}$ implements a residual (skip) connection.

\textbf{CWN.}
We implemented the layer from \cite{bodnar2021weisfeilercell}:

\begin{align*}
\mathbf{x}_\sigma^C &\gets \text{MLP} \Big( \mathbf{x}_\sigma^C + \frac{1}{|\sigma|} \sum_{i \in \sigma} \mathbf{x}_i^N \Big), \\
\mathbf{x}_i^N &\gets \mathbf{W} \mathbf{x}_i^N + 
\frac{1}{|\{(j,\sigma): i,j \in \sigma\}|} 
\sum_{\substack{(j,\sigma) \\ i,j \in \sigma}} 
\text{MLP} \big( \mathbf{x}_i^N + \mathbf{x}_j^N + \mathbf{x}_\sigma^C \big).
\end{align*}

\paragraph{sCWN.} 
This model is a simple boundary, co-boundary aggregation. 
Most of the weights are used to update clique representation, while node representations are updated from the average of clique features. On maximal clique complexes, the model is implemented as follows:
\begin{align*}
\mathbf{x}_\sigma^C &\gets \text{MLP} \Big( \mathbf{W} \mathbf{x}_\sigma^C + \frac{1}{|\sigma|} \sum_{i \in \sigma} \text{MLP}(\mathbf{x}_i^N) \Big), \\
\mathbf{x}_i^N &\gets \mathbf{W} \mathbf{x}_i^N + \frac{1}{|\{\sigma \in \mathcal{X}: i \in \sigma\}|} \sum_{\sigma \ni i} \mathbf{x}_\sigma^C.
\end{align*}

\textbf{fCWN.}
This model mixes an sCWN model with a graph aggregation layer at the node level.
On maximal clique complexes, the model is implemented as follows:
\begin{align*}
\mathbf{x}_\sigma^C &\gets \text{MLP} \Big( \mathbf{W} \mathbf{x}_\sigma^C + \frac{1}{|\sigma|} \sum_{i \in \sigma} \text{MLP}(\mathbf{x}_i^N) \Big), \\
\mathbf{x}_i^N &\gets \mathbf{W} \mathbf{x}_i^N + \mathbf W \dfrac{1}{|\mathcal N_i|}\sum_{j \in \mathcal N_i} x_j +\frac{1}{|\{\sigma \in \mathcal{X}: i \in \sigma\}|} \sum_{\sigma \ni i} \mathbf{x}_\sigma^C.
\end{align*}

\textbf{SCCN.} We used the TopoModelX~\citep{hajij2024topoxsuitepythonpackages} implementation of the SCCN layer from~\citep{yang22a}.

\textbf{Global architecture}.Each model begins with a layer normalization of the input. Each subsequent layer is composed as follows:  
\[
\texttt{Conv}\;\to\; \texttt{LayerNorm/Batchnorm} (\text{with or without})  \;\to\; \texttt{ReLU} \; \; \to \; \texttt{Dropout}.
\]
{Where Conv can be replaced with any convolutional layer under evaluation (\emph{e.g.} sCWN, SCCN, GAT, etc.).}

{\textbf{Graph models.} We experiment with several standard graph neural networks: Simple Graph Convolution (SGC), Graph Convolutional Network (GCN), GraphSAGE, Graph Attention Network (GAT), and Graph Isomorphism Network (GIN). For SGC, we use a modified version with shift operator \(\mathbf S := \mathbf D^{-1}\mathbf A\), concatenating \(\mathbf{x}, \mathbf S\mathbf{x}, \dots, \mathbf S^K \mathbf{x}\) and feeding the result into an MLP. For the other models, we use the PyTorch Geometric implementations with standard hyperparameters.}

\textbf{Node classification.} The final layer applies a convolution followed by Softmax.  

\textbf{Graph classification.} The final layer applies a convolution followed by a global add pooling operation to aggregate node features into a graph-level embedding. Then, it is followed by Softmax.

\begin{figure}[t]
    \centering
    \begin{subfigure}{0.45\textwidth}
        \centering
        \includegraphics[width=\linewidth]{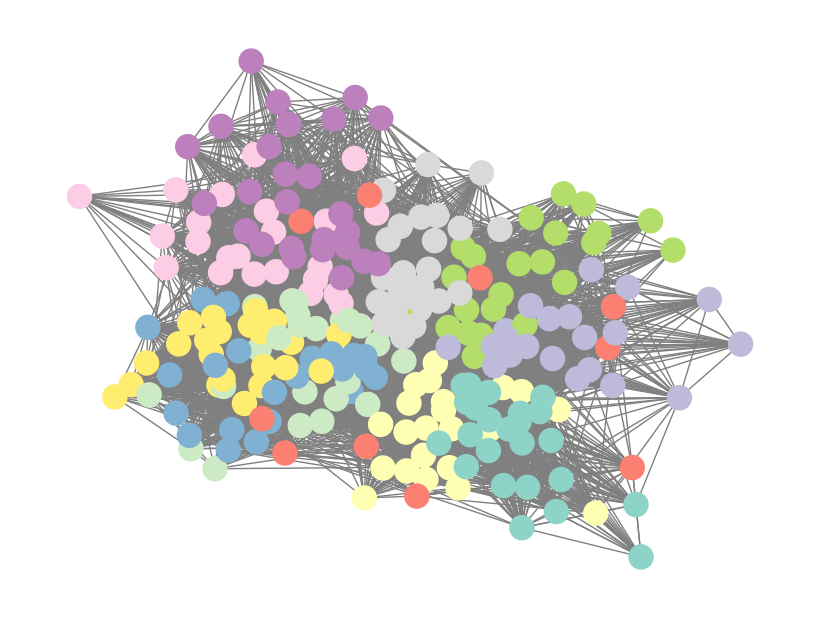}
        \caption{}
        \label{fig:contact-primary-school}
    \end{subfigure}
    \hspace{5pt}
    \begin{subfigure}{0.45\textwidth}
        \centering
        \includegraphics[width=\linewidth]{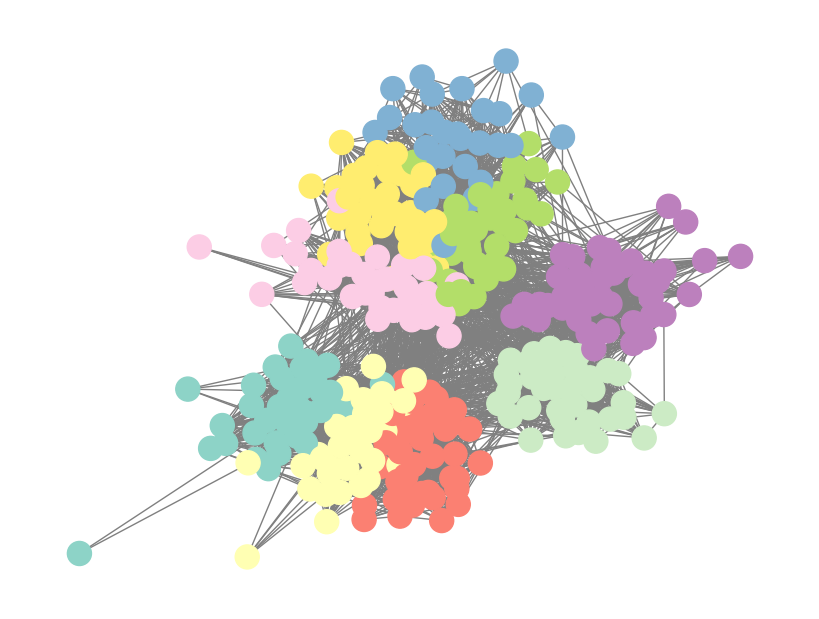}
        \caption{}
        \label{fig:contact-high-school}
    \end{subfigure}
    \caption{Projected datasets: (a) contact-primary-school and (b) contact-high-school.}
    \label{fig:contact-datasets}
\end{figure}

\section{Compute Resources}

All experiments were conducted using NVIDIA GPUs, primarily \textbf{RTX 3090}, \textbf{L40S} and \textbf{H100}.

\textbf{Node classification.}
Node classification experiments are relatively lightweight and computationally efficient. Individual training runs typically complete within minutes to a few hours depending on the dataset and hyperparameter configuration. As a result, full hyperparameter sweeps across all node classification benchmarks can be completed within a few GPU-days.

\textbf{Graph classification.}
Graph classification experiments are more computationally demanding due to the need for repeated training across multiple datasets and hyperparameter configurations. A full sweep over all graph classification datasets typically requires approximately between $1$ hour to $1$ GPU-day.

\textbf{Total compute budget.}
Taking into account all reported experiments, including hyperparameter searches and repeated runs for statistical significance, we estimate the total compute usage to be on the order of \textbf{two months of single-GPU time}. This estimate includes both successful runs and preliminary experiments that informed the final experimental design.

\end{document}